\documentclass[twoside]{article}

\usepackage[accepted]{aistats2024}
\usepackage{tikz}
\usepackage[round]{natbib}
\usepackage[T1]{fontenc}    
\usepackage{hyperref}       
\usepackage{url}            
\usepackage{booktabs}       
\usepackage{amsfonts}       
\usepackage{nicefrac}       
\usepackage{microtype}      
\usepackage{xcolor}         
\usepackage{algorithm, algorithmic}

\usepackage{comment}
\usepackage{microtype}
\usepackage{graphicx}
\usepackage{caption}
\usepackage{subcaption}
\usepackage{wrapfig}
\usepackage[frozencache,cachedir=.]{minted}

\usepackage{amsmath}
\usepackage{amssymb}
\usepackage{mathtools}
\usepackage{amsthm}
\usepackage{dsfont}
\usepackage{multirow}
\usepackage{multicol}

\theoremstyle{plain}
\usepackage{tcolorbox}
\newtheorem{theorem}{Theorem}[section]

\newtheorem{remark}[theorem]{Remark}

\theoremstyle{definition}
\newtheorem{definition}[theorem]{Definition}

\newcommand{\mf}[1]{{\textcolor{black}{#1}}}
\newcommand{\nmf}[1]{{\textcolor{black}{#1}}}

\definecolor{tabblue}{HTML}{5778a4}

\definecolor{taborange}{HTML}{e49444}

\definecolor{tabgreen}{HTML}{6a9f58}

\newcommand{\mlp}{\mathrm{MLP}}

\begin{document}

%
\runningtitle{Neural McKean-Vlasov Processes}
%

\twocolumn[

\aistatstitle{Neural McKean-Vlasov Processes: \\ Distributional Dependence in Diffusion Processes}

\aistatsauthor{ Haoming Yang$^*$ \And Ali Hasan$^{*\dagger}$ \And  Yuting Ng \And Vahid Tarokh}

\aistatsaddress{ Duke University } ]
\def\thefootnote{*}\footnotetext{Equal contribution, junior author listed first.}
\def\thefootnote{\arabic{footnote}}
\def\thefootnote{$\dagger$}\footnotetext{Correspondence: \href{mailto:ali.hasan@duke.edu}{ali.hasan@duke.edu}}\def\thefootnote{\arabic{footnote}}

\begin{abstract}
McKean-Vlasov stochastic differential equations (MV-SDEs) provide a mathematical description of the behavior of an infinite number of interacting particles by imposing a dependence on the particle density.
As such, we study the influence of explicitly including distributional information in the parameterization of the SDE.
We propose a series of semi-parametric methods for representing MV-SDEs, and corresponding estimators for inferring parameters from data based on the properties of the MV-SDE.
We analyze the characteristics of the different architectures and estimators, and consider their applicability in relevant machine learning problems.
We empirically compare the performance of the different architectures and estimators on real and synthetic datasets for time series and probabilistic modeling.
The results suggest that explicitly including distributional dependence in the parameterization of the SDE is effective in modeling temporal data with interaction under an exchangeability assumption while maintaining strong performance for standard It\^o-SDEs due to the richer class of probability flows associated with MV-SDEs. 
\end{abstract}

\section{Introduction}
\begin{figure}
\centering
     \begin{subfigure}{0.44\linewidth}
         \centering
     \includegraphics[width=\textwidth, trim=10pt 10pt 10pt 10pt, clip]{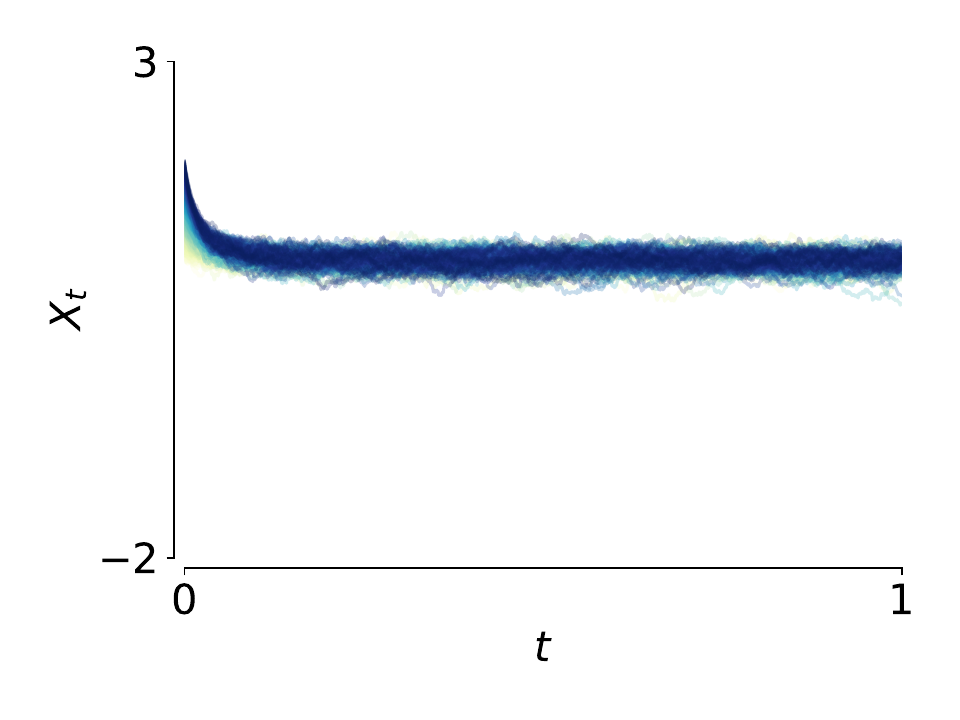}
         \caption{It\^o-SDE}
         \label{fig:simple_diffusion}
     \end{subfigure}
     \begin{subfigure}{0.44\linewidth}
         \centering
    \includegraphics[width=\textwidth, trim= 10pt 10pt 10pt 10pt, clip]{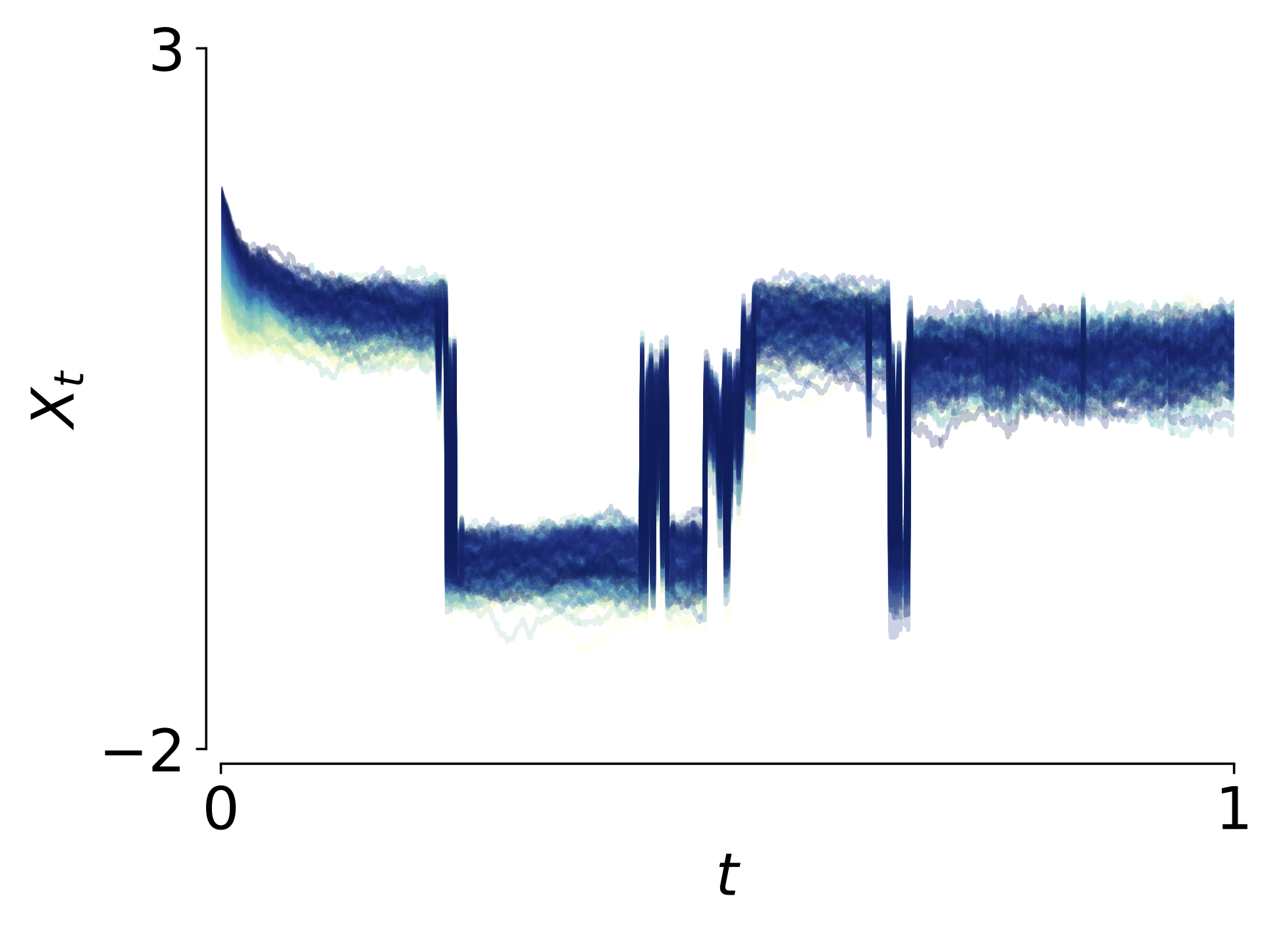}
         \caption{MV-SDE}
         \label{fig:mv_diffusion}
     \end{subfigure}
        \caption{SDE sample paths of a double-well potential, where the particles (a) do not interact and (b) exhibit complex phase transitions as a result only of interaction through weak attraction. 
        }
\label{fig:diffusion_compare}
\end{figure}
Stochastic differential equations (SDEs) model the evolution of a stochastic process through two functions known as the \emph{drift} and \emph{diffusion} functions. 
With It\^o-SDEs, individual sample paths are assumed to be independent, and neural representations of the drift and diffusion have achieved high performance in many applications such as in time series and generative modeling~\citep{song2020score, tashiro2021csdi}.

On the other hand, interacting particle systems are also used to model stochastic processes; they share many characteristics with It\^o-SDEs but additionally dictate an interaction between the different sample paths~\citep{liggett1997ips}.
This has been used as an interpretation of the dynamics of machine learning algorithms, most notably for the attention mechanism in the transformer architecture (e.g. in~\citep{sander2022sinkformers}).
When the number of particles approaches infinity, these processes generalize It\^o-SDEs to \emph{nonlinear SDEs} known as McKean-Vlasov SDEs (MV-SDEs). 
The nonlinearity arises from the individual particle's dependence on the whole particle density, often in the form of a \emph{mean-field} term represented by an expectation with respect to the particle density.
This distributional dependence allows for greater flexibility in the time marginal distributions that the MV-SDE can represent versus the It\^o-SDE. 
An example of the differences between the two frameworks is illustrated in Figure~\ref{fig:diffusion_compare} where  Figure~\ref{fig:simple_diffusion} depicts an It\^o-SDE where the sample paths are independent and Figure~\ref{fig:mv_diffusion} depicts a MV-SDE where the sample paths interact through distributional dependence. 
While these models appear in a variety of disciplines such as in finance~\citep{feinstein2021dynamic}, biology~\citep{Keller1971KSChemo}, and social sciences~\citep{carrillo2020long}, relatively few works have considered the application in machine learning tasks.
Moreover, representing such processes requires careful consideration, since the drift and diffusion parameterizations must now include \emph{distributional information} rather than just the state as inputs. 

This brings us to a motivating question:
\begin{tcolorbox}
\begin{center}
\unskip
\textit{(Q1) Can we develop neural architectures that satisfy the properties of MV-SDEs?}
\end{center}
\end{tcolorbox}
\noindent
To answer \emph{(Q1)}, we consider two ideas: (i) expressing a layer in a neural network as an expectation under a change of measure and (ii) using generative models to capture distributions and compute empirical expectations.   
Our second question is concerned with how the theoretical generality of MV-SDEs to It\^o-SDEs occurs in practice:
\begin{tcolorbox}
    \begin{center}
\unskip
\textit{(Q2) Does explicitly including distributional dependence improve modeling capabilities?}
\end{center}
\end{tcolorbox}
\noindent
This question is related to the initial insight that the MV-SDE can be viewed as a generalization of powerful architectures, such as the transformer, and we are interested in whether adding distributional dependence improves modeling capabilities.
We discuss a few theoretical properties that motivate this question and answer the question affirmatively by empirically comparing different architectures for applications in time series and in probabilistic modeling.

\subsection{Related work}
Estimating SDEs from data is an important component of many scientific disciplines with a number of methods developed specifically for this task, e.g.~\citet{chen2021solving, hasan2021identifying}.
Methods that estimate MV-SDEs from observations often assume known interaction kernels and drift parameters. 
They then rely on a large number of samples at regularly spaced time intervals to empirically approximate the expectation in the mean-field term~\citep{messenger2021learning,della2022nonparametric,yao22mean,della2023lan}.
In~\citet{pavliotis2022method}, the authors describe a method of moments estimator for the parameters of the MV-SDE. 
Other approaches concerned analyzing the partial differential equation (PDE) associated with MV-SDEs as in~\citet{gomes2019parameter}. 
In our work, we are primarily concerned with inference in regions where we have limited time-marginal data and the number of samples is not large.

In a machine learning context, It\^o-SDEs have been extensively used for problems such as generative modeling~\citep{song2020score, de2021diffusion}.
These models consider a forward SDE which maps data to noise and a reverse SDE which transforms noise to a target distribution, but the forward model notably does not include distributional dependence.
\citet{ansari2020refining} and~\citet{alvarez2021dataset} consider using MV-SDEs for improving generative models and domain adaptation tasks, respectively.
These models again consider applying known MV-SDEs rather than considering the inverse problem of recovering the parameters of the MV-SDE. 
Closely related to the our work,~\citet{pham2022mean} describe two types of neural network architectures that use binning and the empirical measure to represent the parameters of a MV-SDE and describe their performance on estimation tasks when given the mean-field term. 
However, their parameter estimation procedure requires knowledge of the original function being approximated through a known interaction kernel.
Additionally, none of the methods consider a nonparameteric approach to solving for an optimal drift that uses distributional dependence.
Other applications of MV-SDEs in machine learning topics include estimating optimal trajectories in scRNA-Seq data~\citep{chizat2022trajectory} and stochastic control problems relating to mean-field games~\citep{ruthotto2020machine, han2022learning}.
Inverse problems can also be solved by deriving an appropriate MV-SDE as the authors describe in~\citet{crucinio2022solving}.
Extensive analysis of the dynamics of the parameters of a neural network under stochastic gradient descent has been conducted using the theory from MV-SDEs, e.g.~\citep{hu2021mean}.
These methods use a pre-described form of the drift to conduct their analyses whereas we're interested in learning a representation of the drift.

\textbf{Our Contributions}
To address the lack of non-parametric MV-SDE estimators in the literature, this paper contributes the following:
First, we present two neural architectures for representing MV-SDEs based on learned measures and generative networks; then, we present two estimators, based on maximum likelihood, used in conjunction with the architectures without prior knowledge on the structure of the drift; next, we characterize the properties of implicit regularization and richer probability flows of these architectures; finally, we empirically demonstrate the applicability of the architectures on time series and generative modeling.

\section{Properties of MV-SDEs}
\label{sec:properties}
We begin by describing the background and properties of the transition densities of MV-SDEs. 
Figure~\ref{fig:properties} illustrates some of these concepts qualitatively where we first consider non-local dynamics and then consider jumps in the sample paths.

\subsection{Background}
Consider a domain $\mathcal{D} \subset \mathbb{R}^d$ and let $\mathcal{P}_k(\mathcal{D})$ be the space of all probability distributions supported on $\mathcal{D}$ with finite $k$th moment. 
Let $W_t \in \mathbb{R}^d$ be a $d$-dimensional Wiener process and let $X_t \in \mathbb{R}^d$ be a solution to the following MV-SDE
\begin{equation}
    \mathrm{d}X_t = b(X_t, p_t, t) \mathrm{d}t + \sqrt{\Sigma(X_t, p_t, t)} \mathrm{d} W_t
    \label{eq:mv_sde}
\end{equation}
where $p_t$ denotes the law of $X_t$ at time $t$ and $\sqrt{\Sigma}$ denotes the Cholesky decomposition of $\Sigma$.
We denote the \emph{drift} as $b : \mathbb{R}^d \times \mathcal{P}_2(\mathcal{D}) \times \mathbb{R}_+ \to \mathbb{R}^d$ and the \emph{diffusion} as $\Sigma : \mathbb{R}^d \times \mathcal{P}_2(\mathcal{D}) \times \mathbb{R}_+  \to \texttt{SPD}(\mathbb{R}^{d \times d})$ with \texttt{SPD} denoting the space of symmetric, positive definite matrices.

We focus on the case where the diffusion coefficient is a known constant, $\sigma$, and focus on estimating the drift, $b$, from data.
In addition, for simplicity in analysis, we suppose that $b$ factors linearly into a \nmf{non-interacting component}, and an \mf{interacting component}, where the mean-field term with dependence on $p_t$ is often written in terms of an expectation, specifically 
\begin{equation}
\mathrm{d}X_t = \nmf{f(X_t,t)}\mathrm{d}t + \mf{\mathbb{E}_{y_t\sim p_t}\left [\varphi\left(X_t,y_t\right)\right]}\mathrm{d}t + \sigma\mathrm{d}W_t
\label{eq:linear_drift}
\end{equation}
where $f:\mathbb{R}^d\times\mathbb{R}_+\to\mathbb{R}^d$ can be seen as the It\^o drift, the expectation as the mean-field drift, and $\varphi: \mathbb{R}^d\times \mathbb{R}^d \to \mathbb{R}^k$ as the \emph{interaction} function describing the interaction between particles, e.g. attraction with $\varphi(x,y)=-(x-y)$ in~Figure~\ref{fig:mv_diffusion} and the left side of Figure~\ref{fig:properties}.
This factorization is used in many practical models of interest, such as in bacterial chemotaxis, neural oscillators, and social networks~\citep{carrillo2020long}. 
Note that the proposed architectures do not assume this factorization, we use this factorization for ease in exposition.
As mentioned, unlike It\^o-SDEs which only depend on the state of the current particle $X_t$ and $t$, MV-SDEs additionally depend on the marginal time distribution of particles $p_t$. 
By introducing a dependence on the marginal law, the transition density of the process satisfies a richer class of functions.
To motivate this further, we next describe the properties of non-local dynamics and jumps.

\begin{figure}
    \centering
\includegraphics[width=0.23\textwidth, trim={15pt 15pt 15pt 15pt}]{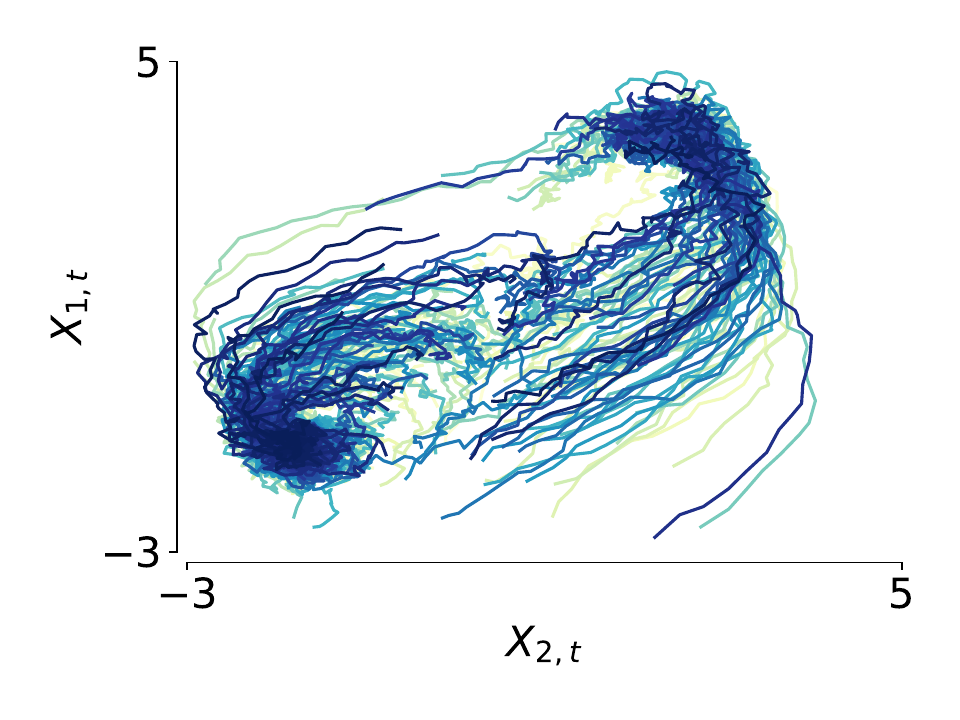}
\includegraphics[width=0.23\textwidth, trim={15pt 15pt 15pt 15pt}]{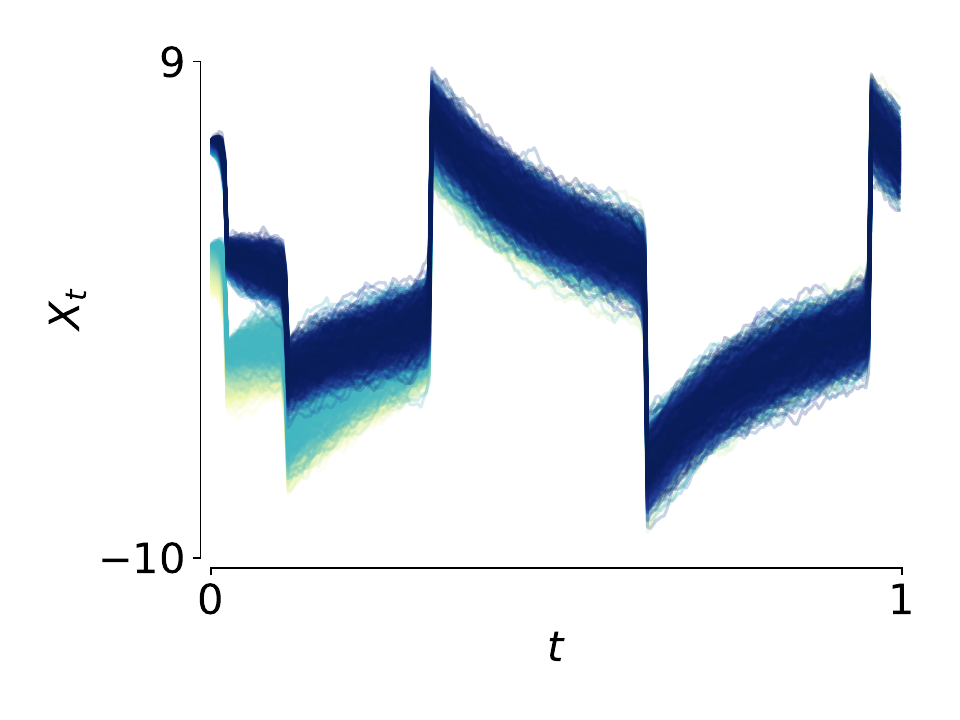}
    \caption{MV-SDE sample paths with non-local dynamics (left) and discontinuities (right).}
    \label{fig:properties}
\end{figure}
\subsection{Non-locality of the transition density}
Following the background, we describe a favorable property of the MV-SDE that induces non-local dependencies in the state space.  
The transition density of~\eqref{eq:linear_drift} can be written as the non-linear PDE
\begin{align}
  \nonumber & \partial_t p_t(x) =   \\ &- \nabla \cdot \left ( \underbrace{\phantom{\frac{}{}}p_t f(x)\phantom{\frac{}{}}}_{\text{It\^o Drift}} + \underbrace{p_t\int \varphi(x, y_t) p_t(y_t) \mathrm{d}y_t} _{\text{Non-Local Interactions}} - \underbrace{\frac{\sigma^2}{2} \nabla p_t}_{\text{Diffusion}} \right ).\label{eq:pde}
\end{align}
This non-local behavior has a variety of implications. For example, the distribution of particles ``far away'' from a reference particle can affect the behavior of the reference particle.
This property is illustrated in the left side of Figure~\ref{fig:properties} with an example from the mean-field FitzHugh-Nagumo model used to model spikes in neuron activation, leading to interactions between distinct spikes~\citep{crevat2019rigorous}.
Notably, this is not possible when considering only the It\^o drift, since that operator acts locally on the density.

\subsection{Discontinuous sample paths}
\label{sec:discSamplePaths}
The richer class of densities modeled by MV-SDEs has direct influence on individual sample paths.
In a modeling scenario, we may wish to approximate a process that exhibits jumps. 
For example, in finance, a number of related entities may have common exposure and experience failure simultaneously~\citep{nadtochiy2019particle, feinstein2021dynamic}.
Similarly, in neuroscience, a number of neurons spiking simultaneously results in discontinuities in the sample paths~\citep{carrillo2013classical}.
The fact that the interaction of many particles can cause blowups leads to a remarkable property of MV-SDEs that allows discontinuous paths.
The major benefit of this property is that we do not need to consider an additional jump noise process -- we only need to specify a particular interaction between the particles to induce the jump behavior.
We can then represent point processes through the discontinuous points of the sample path, which provides a framework of interpreting a point process through particle interactions. 
Similar techniques have been developed in the case of It\^o-SDEs, where the decomposition of individual sample paths leads to a point process structure and the sample paths can then be interpreted for downstream application~\citep{hasan2023inference}.
A simple proof for the case of positive feedback is given in~\citet[Theorem 1.1]{hambly2019mckean}.
Intuitively, the jumps occur due to the drift inducing a simultaneous force against the particles, leading to the discontinuity.

\section{Mean-field architectures}
\begin{figure*}
\centering
    \begin{tikzpicture}

\node at (1, 0) [rectangle,minimum size=2cm] {
  \shortstack{Implicit Measure (IM) \\ $\frac1{n}\sum_{i=1}^{n}\varphi(\,\cdot\,, W_0^{(i)})\frac{\mathrm{d}p_t}{\mathrm{d}p_0}$ \\ $\phantom{ \left\{\hat{X}_t^{(i)}\right\}_{i=1}^K \sim \hat{P}_t}$}
  };
  \node at (5.5, 0) [rectangle,minimum size=2cm] {
  \shortstack{Empirical Measure (EM) \\ $\frac1n\sum_{i=1}^{n}\varphi(\,\cdot\,, X_t^{(i)})$ \\ $ \left\{X_t^{(i)}\right\}_{i=1}^n \sim p_t$}
  };
 \node at (9.5, 0) [rectangle, minimum size=2cm] {
  \shortstack{ Marginal Law (ML) \\ 
  $\frac1{n}\sum_{i=1}^{n}\varphi(\,\cdot\,, \hat{X}_t^{(i)})$ \\  $\left\{\hat{X}_t^{(i)}\right\}_{i=1}^{n} \sim \hat{p}_t$ }
  };
\draw [stealth-stealth](-1.5,-1) -- (11.5,-1) node [pos=0.95,below,font=\footnotesize] {Explicit $p_t$ } node [pos=0.05,below,font=\footnotesize] {Implicit $p_t$};
\end{tikzpicture}
\caption{Schematic comparing neural architectures for modeling MV-SDEs. Implicit measure (IM) architecture uses a mean-field layer that represents particles as learned weights and computes the expectation under a learned change of measure; the empirical measure (EM) architecture computes the expectation with the observed particles; the marginal law (ML) architecture learns the particle density, and computes an empirical expectation with samples from the learned density.
}
\label{fig:arch}
\end{figure*}
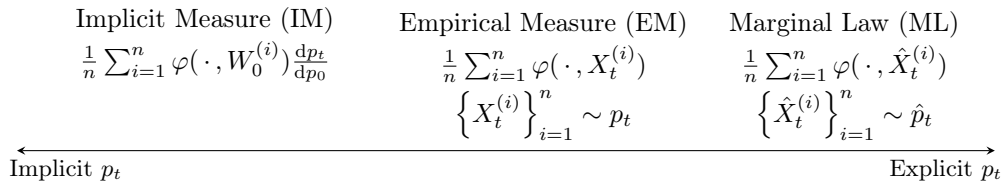

\label{sec:architectures}
We first consider a modification of the architecture proposed in~\citet{pham2022mean} that empirically computes the expectation using observations, and denote it as the empirical measure (EM) architecture. 
We then propose two architectures -- an architecture based on representing a learned measure with neural weights, denoted as the implicit measure (IM) architecture, and a generative architecture based on representing the marginal law of the samples, denoted as the marginal law (ML) architecture. 
Figure~\ref{fig:arch} provides a schematic of the mean-field representations with different architectures.
To motivate modeling the expectation in the mean-field term with the EM, IM and ML architectures, we first assume that the drift and diffusion coefficients of the MV-SDE in~\eqref{eq:mv_sde} exhibit sufficient regularity such that the empirical law converges to the true law of the system (i.e. $\frac1n \sum_{i=1}^n \delta_{X_t^{(i)}} \to_{n\to\infty} p_t(X_t)$), i.e. propagation of chaos holds~\citep{meleard1996asymptotic}.
We relax this assumption for the experiments since these are difficult to verify in practice.

\subsection{Empirical measure architecture}
\label{sec:xt}

Suppose we observe $n$ particles at each time $t$ given by $\{X_t^{(i)}\}_{i=1}^n$ and denote the discrete measure associated with these observations as $p_t^\delta = \frac1n \sum_{i=1}^n\delta_{X_t^{(i)}}$.
Then, we can use $p_t^{\delta}$ to approximate the expectation in~\eqref{eq:linear_drift} as
\begin{align}
    \nonumber \mathbb{E}_{y_t\sim p_t}\left [\varphi\left(X_t,y_t\right)\right] &\approx \mf{\mathbb{E}_{y_t \sim p_t^{\delta}}\left [ \varphi\left(X_t, y_t ; \theta \right) \right ]} \\ &= \frac{1}{n}\sum_{i=1}^n \varphi\left(X_t, X_t^{(i)} ; \theta \right)
\label{eq:xt}
\end{align}
for a neural network $\mf{\varphi(\cdot, \cdot; \theta)}$ describing the interaction function between the particles~\citep{pham2022mean}. 
Suppose the non-mean field component $f$ is also represented with a neural network $\nmf{f(\cdot, t; \theta)}$. 
Assuming that $\mf{\varphi}$ and $\nmf{f}$ are well learned, this architecture can represent the MV-SDE drift in the limit as the number of observations $n\to \infty$.
We refer to this architecture as the \emph{empirical measure} (EM) architecture since at each time step the expectation is taken with respect to the empirical measure derived from the observations.
We note that the empirical measure may be difficult to obtain at all time steps or may contain few samples, often a smoothing technique is necessary to overcome this issue. 

\subsection{Implicit measure architecture}
\label{sec:w0}
To empirically compute the expectation, the EM architecture in~\eqref{eq:xt} relies on obtaining the empirical measure at each time point.
This may be difficult in practice for a variety of reasons such as having few samples or the lack of data at some time points. Instead, in the implicit measure (IM) architecture, we represent particles as learned weights and compute the expectation under a learned change of measure.

Let us first recall that a single layer in a multilayer perceptron $(\mlp)$ can represent an expectation as
\begin{equation}
\mlp^{W,b}(x) = \int \sigma \left ( Wx + b \right ) \mathrm{d}\nu\left(W,b\right) 
\label{eq:mlp}
\end{equation}
where the expectation is taken over $\nu \left ( \cdot \right )$, a measure over the space of parameters $y = ( W, b)$, and $\sigma$ is an activation function. 

When $\nu=\frac{1}{n}\sum_{i=1}^n\delta_{y^{(i)}}$, a discrete measure with $n$ particles, the expectation is exactly a single layer of width $n$, suggesting a correspondence between an empirical measure with $n$ samples and a single layer of width $n$.
Building on this correspondence, we propose a mean-field layer:

\begin{definition}[Mean-field Layer]
\label{def:mf}
Define the weight of the mean-field layer with width $n$ as the matrix $W_0 \in \mathbb{R}^{n \times d}$ and denote its $i$th row as $W_0^{(i)}$.
The mean-field layer then is defined by the operation 
\begin{equation}
\mf{\mathrm{MF}_{(n)}(\varphi(X_t)) := \frac1n \sum_{i=1}^n \varphi(X_t, W_0^{(i)})\frac{\mathrm{d}p_t}{\mathrm{d}p_0}.}
\label{eq:mf_module}
\end{equation}
\end{definition}

 The mean-field layer (MF) can be thought of as another layer within the network architecture that approximates the law $p_t$. 
 Each row $W_0^{(i)}$ is of size $\mathbb{R}^d$, corresponding to the dimensions of $X_t^{(i)}\in\mathbb{R}^d$. 
 The activation function of the mean-field layer is the average over the augmented dimension over which MF operates.
 The change of measure $\frac{\mathrm{d}p_t}{\mathrm{d}p_0}$, such that $\mathbb{E}_{p_t}[\,\cdot\,]=\mathbb{E}_{p_0}[\,\cdot\, \frac{\mathrm{d}p_t}{\mathrm{d}p_0}]$, can be learned as part of the estimator of the interaction function, $\varphi(\cdot,\cdot,t;\theta)$. Thus
 \begin{equation}
\nonumber \mathbb{E}_{y_t\sim p_t}\left [\varphi\left(X_t,y_t\right)\right] \approx \frac1n \sum_{i=1}^n \varphi(X_t, W_0^{(i)},t;\theta).
\end{equation}
 Importantly, the above representation allows modeling mean-field interactions without the need for a full set of observations at each time point and without the need to explicitly represent the distribution $p_t$ at each time point. 
 Assuming that $\varphi$ and MF are well learned, this architecture can represent the true MV-SDE drift in the limit as the width $n\to \infty$. 
 We note empirically that a finite $n$ is sufficient and we provide examples of ablations in the appendix.

A similar analysis can be made for the standard MLP architecture.
However, the explicit separation of $f$ and $\varphi$ is not expressed and $\varphi$ may take a simpler convolutional form. 
This leads us to the following remark:
\begin{remark}\label{remark:itosde}
From the above discussion, the expectation of the form $\mathbb{E}_{y\sim p_t}[\varphi(X_t - y)]$ with respect to the law $p_t$ may be implicitly represented by a MLP. 
 \end{remark}

 Our motivation is then concerned with how a relatively more explicit distribution dependence with $\varphi$ and MF affect modeling capabilities. 
 This explicit structure lends to an implicit regularization that promotes a smaller norm of the mean-field component under a maximum likelihood estimation framework, which we detail later in Section~\ref{sec:implicit}.

\subsection{Marginal law architecture}
\label{sec:nf}

A solution to the MV-SDE is the pair $(X_t,p_t)$ such that $p_t=\text{Law}(X_t)$. 
For this reason, $p_t$ is often itself the main object of study. 
In the marginal law (ML) architecture, in conjunction with the drift, we introduce a generative model for representing the time-varying density.
In this case, we approximate the expectation in~\eqref{eq:linear_drift} as
\begin{align}
    \nonumber \mathbb{E}_{y_t\sim p_t}\left [\varphi\left(X_t,y_t\right)\right] &\approx \mf{\mathbb{E}_{y_t \sim \hat{p}_t}\left [ \varphi\left(X_t, y_t ; \theta \right) \right ]} \\ &= \frac{1}{n}\sum_{i=1}^n \varphi\left(X_t, \hat{X}_t^{(i)} ; \theta \right)
\label{eq:pt}
\end{align}
where the expectation is taken with respect to the discrete measure derived from samples $\{\hat{X}_t^{(i)}\}_{i=1}^n$ from the generative model $\hat{p}_t(\cdot;\phi)$.

In addition, we regularize $\hat{p}_t(\cdot;\phi)$ such that the marginal law is consistent with the flow relating to the drift, which we detail later in Section~\ref{sec:est_marginallaw} using the PDE in~\eqref{eq:pde}. 

\paragraph{A PDE representation of $p_t$} 
The transition density of the MV-SDE is known to satisfy the corresponding nonlinear Fokker-Planck equation under assumptions stated above~\citep{pavliotis2022method, belomestny2019iterative}.
If we consider the $X_t$ that satisfies the MV-SDE and the $p_t$ which is the law of the samples at time marginal $t$ with initial condition $p_0 = u(x) \in \mathcal{P}_2(\mathcal{D})$.
Then, $p_t$ is a solution of
\begin{equation}
\label{eq:nl_fp}
\partial_t p_t(x) = -\nabla \cdot \left (  B(x,p_t) p_t \right) + \frac12 \sigma \nabla^2 p_t, \quad p_0 = u(x).
\end{equation} 
Moreover,~\eqref{eq:nl_fp} has a solution given in terms of an expectation over sample paths $X_t$ as
\begin{equation}
\label{eq:fk}
p_t(x) = \mathbb{E}\left[e^{\int_0^t-\nabla \,\cdot \, B(x_s,p_{s}) \mathrm{d}s}p_0(X_t) \mid X_0 = x\right].
\end{equation}
As an aside, note that there are a number of relevant PDEs can be written in this form, e.g. the Burgers equation from fluid dynamics which is given by $\partial_t u_t = -u \partial_x u + \partial_{xx}u$ is obtained with $B \equiv \int \mathds{1}_{X_t - y > 0} p_t(y) \mathrm{d}y $~\citep{bossy1997stochastic}.
This stochastic representation lends itself to an efficient method for computing solutions to these PDEs in high dimensional settings.
Traditional solvers such as finite differences or finite elements do not scale to high dimensions, making an approach such as this appropriate for finding the solution of the PDE.

\section{Parameter estimation}
Having presented the neural architectures, we now present 
  estimators, based on maximum likelihood,  used in conjunction with the architectures without prior knowledge on the structure the drift. We first describe the likelihood function for use in cases with regularly sampled data. We then describe a bridge estimator for cases of irregularly sampled data. 
In addition, we describe an estimator for the generative architecture based on both the likelihood function and the transition density.
For this section, we assume that we observe multiple paths, i.e., $\left\{\{X_{t_j}\}_{j=1\ldots K}^{(i)}\right\}_{ i=1\ldots N}$. Full details of all algorithms are in the appendix.

\subsection{Maximum likelihood estimation}
We use an estimator based on the path-wise likelihood derived from Girsanov's theorem and an Euler-Maruyama discretization for the likelihood, considered in~\citet{sharrock2021parameter}. 
The likelihood function is given as
\begin{align}
   \nonumber \mathcal{L}(\theta; t_1, t_K) :=  \exp \bigg (&\frac{1}{\sigma^2}\int_{t_1}^{t_K} b\left(X_s, p_s,  s; \theta\right) \mathrm{d} X_s  \\ - &\frac{1}{2\sigma^2} \int_{t_1}^{t_K} b\left(X_s, p_s, s; \theta\right)^2 \mathrm{d} s \bigg),
    \label{eq:likelihood}
\end{align}
where $b$ is the unknown drift represented as one of the presented architectures and $\sigma$ is the diffusion coefficient in~\eqref{eq:mv_sde} and~\eqref{eq:linear_drift}. Following discretization, with the approximations $\Delta X_{t_j} = X_{t_{j+1}} - X_{t_j}$ and $\Delta{t_j}=t_{j+1}-t_j$, the log-likelihood is approximated by
\begin{align*}
 \log \, \mathcal{L}(\theta; t_1, t_K) \approx \sum_{j=1}^{K-1} b\left(X_{t_j},p_{t_j}, t_j; \theta\right) ( X_{t_{j+1}} - X_{t_j}) \\ - \frac12 \sum_{j=1}^{K-1} b\left(X_{t_j}, p_{t_j},  t_j; \theta\right)^2 ( t_{j+1} - {t_j}).
\end{align*}
Optimization is performed using standard gradient based optimizers with the drift $b$ represented as one of the presented architectures. 

\subsection{Estimation with Brownian bridges}
\label{sec:est_bb}

Often data are not collected at uniform intervals in time, but rather, the time marginals may be collected at irregular intervals. 
In that case, we consider an interpolation approach to maximizing the likelihood building on the results of~\citet{lavenant2021towards} and~\citet{cameron2021robust} in the It\^o-SDE case. 
We can write the likelihood conditioned on the set of observations (dropping the particle index for ease of notation) as 
\begin{align*}
\mathcal{L}_{BB}(\theta) = \mathbb{E}_\mathbb{Q}\left[\prod_{j=1\ldots K-1}\mathds{1}_{\{ Z_{t_{j+1}} = X_{t_{j+1}}\} }\mathcal{L}(\theta; t_{j}, t_{j+1}) \right ]
\end{align*}
where $\{Z_s: s \in [t_j, t_{j+1}]\}$ is a Brownian bridge from $X_{t_j}$ to $X_{t_{j+1}}$ and $\mathbb{Q}$ is the Wiener measure.
Brownian bridges can easily be sampled and reused for computing the expectation.
 By applying Jensen's inequality, we can write an evidence lower bound (ELBO) as 
\begin{align}
   \nonumber &\log \mathcal{L}_{BB} \\ &\geq \mathbb{E}_\mathbb{Q}
\left [ \sum_{j=1\ldots K-1} \log \mathcal{L}(\theta; t_{j}, t_{j+1}) \mathds{1}_{\left\{ Z_{t_j} = X_{t_j}  \right \}_{j=1}^K} \right ].
\label{eq:bb_elbo}
\end{align}
In this case, the estimator aims to fit the observed marginal distributions exactly while penalizing deviations from the Brownian bridge paths in regions without data.

\subsection{Estimation with explicit marginal law}\label{sec:est_marginallaw}

Returning to the ML architecture described in Section~\ref{sec:nf}, where we explicitly model the density $p_t$ with a generative network $\hat{p}_t$, our estimator should enforce the consistency between $\hat{p}_t$ and the flow relating to the drift. We do so using the PDE in~\eqref{eq:pde}. 
Let the parameters of the drift be $\theta$ and the parameters of the generative model be $\phi$, we solve the optimization problem
\begin{align}
\label{eq:nf_mle}
        &\max_{\theta, \phi} \quad \mathbb{E}\left[\mathcal{L}( \theta, \phi \mid \{X_{t_j}\}_{j=1\ldots K})\right] \quad s.t. \\
        &  \int_{t_{j}}^{t_{j+1}} \left| \hat{p}_s(x ; \phi) - \mathbb{E}\left[\hat{p}_{t_{j+1}}\left(\hat{X}_{t_{j+1}} ; \phi\right) \mid \hat{X}_s = x\right]  \right |\mathrm{d} s = 0 
        \label{eq:nf_pde}
\end{align}
for time intervals indexed by $j=1,\ldots, K-1$, state space $ x \in \mathrm{supp}(X_t)$,  and where the trajectories of $\hat{X}_t$ are given by the dynamics of the ML architecture, specifically
\begin{align*}
\nonumber \mathrm{d}\hat{X}_t = \nmf{f(\hat{X}_t,t;\theta)}\mathrm{d}t &+ \mf{\mathbb{E}_{y_t\sim \hat{p}_t(\cdot;\phi)}\left [\varphi\left(\hat{X}_t,y_t;\theta\right)\right]}\mathrm{d}t \\ &+ \sigma\mathrm{d}W_t.
\end{align*}

The likelihood at the observed margins is first maximized in~\eqref{eq:nf_mle}.
In~\eqref{eq:nf_pde}, the marginals at previous times are regularized using the correspondence between the PDE and its associated SDE via the nonlinear Kolmogorov backwards equation~\citep{buckdahn2017mvpde}, which describes $p_t$ as an expectation of trajectories at a terminal time, i.e. $p_t(x)=\mathbb{E}[p_T(X_T) \mid X_t=x]$ for $t<T$. 

\section{Modeling properties}

Having discussed the architectures and estimators, we now discuss specific properties of the modeling framework, which follow from the theoretical discussion presented in Section~\ref{sec:properties}. 
We first discuss how the factorization into $\varphi$ and MF lends to an implicit regularization of the IM architecture. 
We then compare the gradient flows of It\^o-SDEs and MV-SDEs.  We additionally provide more intuition on the proposed architectures in Appendix D.

\subsection{Implicit regularization of the implicit measure architecture}
\label{sec:implicit}

Closely related to the IM architecture are MLP representations of It\^o-SDEs, where we previously remarked may model MV-SDEs. 
On the other hand, the factorization of the IM architecture into an interaction function and a learned measure leads to a type of implicit regularization when the parameters are estimated using gradient descent. 
\begin{figure*}
    \centering
    \includegraphics[width = \textwidth]{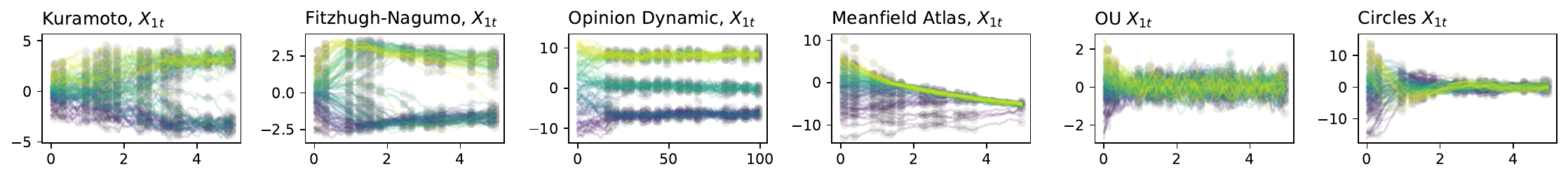}
    \includegraphics[width = \textwidth]{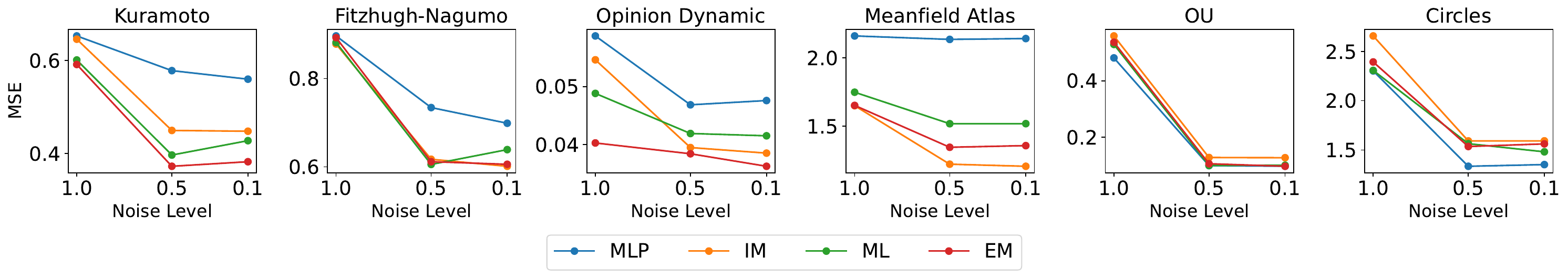}
    \centering
     \begin{subfigure}[t]{0.16\textwidth}
         \centering
         \includegraphics[width=\textwidth]{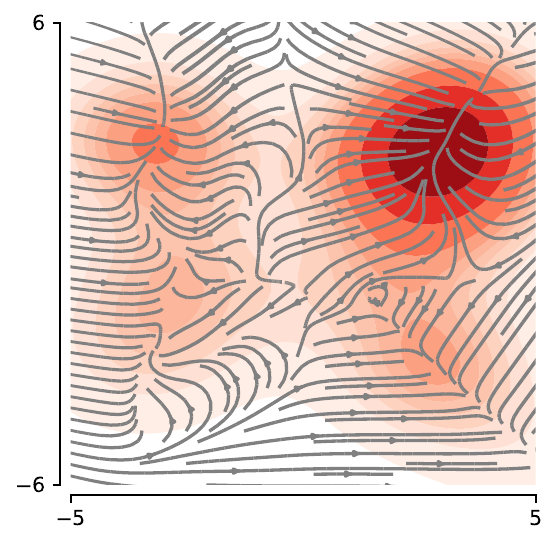}
         \caption{MLP}
         \label{fig:MLP_kura}
     \end{subfigure}
     \hfill
     \begin{subfigure}[t]{0.16\textwidth}
         \centering
         \includegraphics[width=\textwidth]{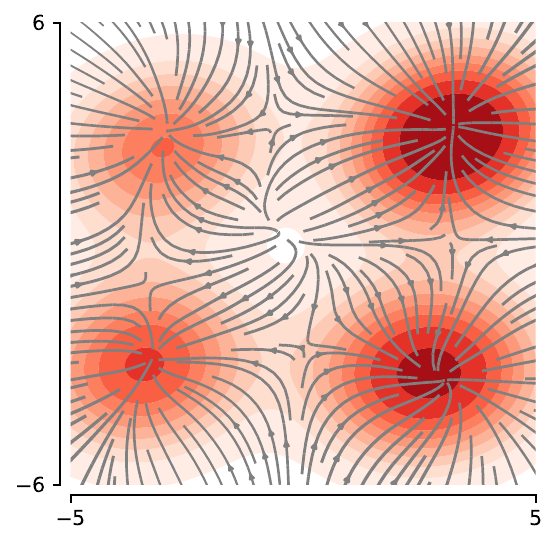}
         \caption{EM}
         \label{fig:Xt_kura}
     \end{subfigure}
     \hfill
     \begin{subfigure}[t]{0.16\textwidth}
         \centering
         \includegraphics[width=\textwidth]{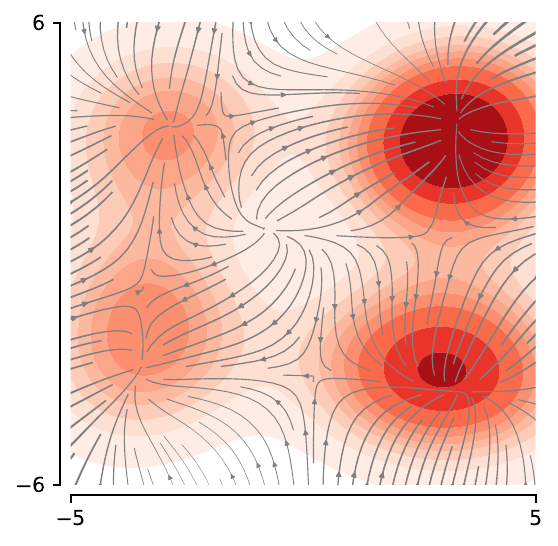}
         \caption{IM}
         \label{fig:W0_kura}
     \end{subfigure}
     \hfill
     \begin{subfigure}[t]{0.16\textwidth}
         \centering
         \includegraphics[width=\textwidth]{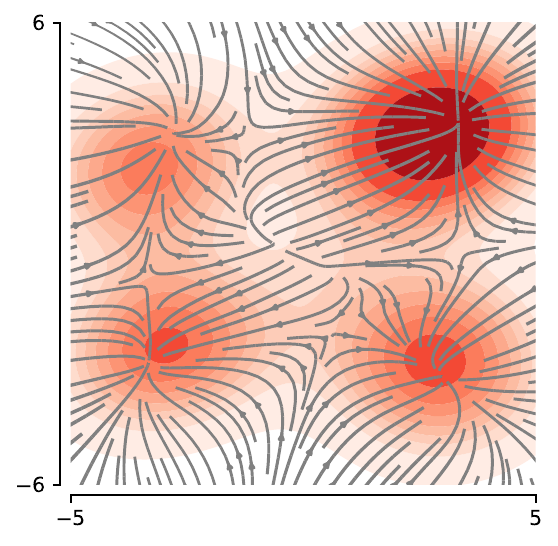}
         \caption{ML}
         \label{fig:gen_kura}
     \end{subfigure}
     \hfill
     \begin{subfigure}[t]{0.16\textwidth}
         \centering
         \includegraphics[width=\textwidth]{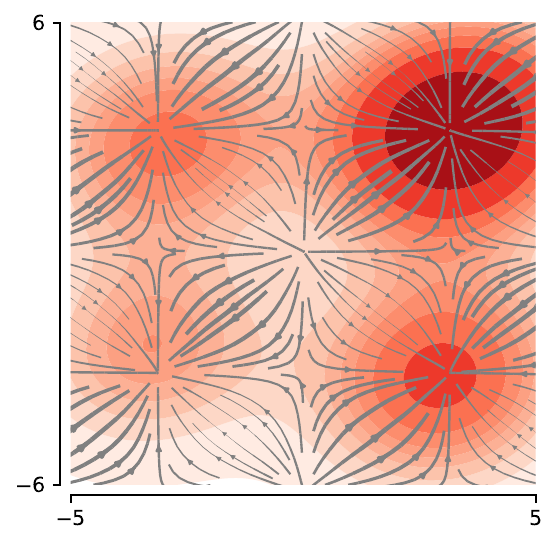}
         \caption{Truth}
         \label{fig:true_kura}
     \end{subfigure}
    \caption{Top row: sample paths from the different synthetic datasets. Middle row: mean squared error (MSE) of different architectures' performance (average of 10 runs) on drift estimation, under the effect of different levels of observation noise. 
    Bottom row: Example of estimated gradient flow of Kuramoto model at terminal time. The colors correspond to the density of generated samples at terminal time.}
    \label{fig:noise_result}
\end{figure*}

\begin{theorem}[Implicit Regularization]
\label{prop:ib}
Suppose $f$ and $\varphi$ are known and fixed. 
Consider a mean-field architecture as described above with $f, \varphi$ known and a linear structure, i.e.
$
B(X_t, p_t, t) =  \mf{\int \varphi(X_t, y) \, \mathrm{d} p_t(y)} + \nmf{f(X_t, t)} .
$
Further, assume that $\varphi$ is twice differentiable. 
Then, for each time step $t$, the minimizing finite width MF with weight matrix $W_0 \in \mathbb{R}^{n \times d}$ and $i$th row $W_0^{(i)}$ under gradient descent satisfies the following optimization problem
\begin{align*}
\min_{W_0} \quad  \sum_{i=1\ldots n}\sum_{j=1\ldots d}\varphi(X_t, W_0^{(i)})_j \quad  \\
\mathrm{s.t.} \quad  \mathbb{E}\left[ \frac1{2 \Delta t} \left \|X_{t+\Delta t} - X_{t} - b(X_{t}, p_{t},t)  \right \|^2 \right] = 0.
\end{align*}
\end{theorem}
\begin{proof}
We follow the blueprint in~\citet{belabbas2020implicit} and give full details in the appendix.
\end{proof}

Theorem~\ref{prop:ib} effectively says that the mean-field system approximated is the one that has the least influence from the other particles. 
In the case where $\varphi$ can be decomposed as a norm, this amounts to finding the drift parameterized by weight $W_0$ with smallest norm while still matching the marginals.

\subsection{Gradient flows of the MV-SDE}
To illustrate the difference between the particle flows of MV-SDEs and It\^o-SDEs, we consider a gradient flow perspective to describe the functionals that are minimized according to the different SDEs~\citep[Section 8.3]{villani2021topics}.
In the following remark, we apply this idea to a gradient flow that minimizes the energy distance. 

\begin{remark}[Minimizing the Energy Distance]
    Consider two densities $p_t,q$ such that $p_t \ll q$ for all $t$.
    The gradient flow induced by the MV-SDE with the drift
    \begin{align*}
    b(X_t, p_t, t) = &\mathbb{E}_{ y_t \sim p_t} \bigg[ \nabla \bigg ( 2\| X_t - y_t  \|\frac{\mathrm{d} q}{\mathrm{d}p_t} \\ &-\| X_t -y_t \|^2 - \| X_t - y_t  \|^2\left(\frac{\mathrm{d}q}{\mathrm{d}p_t}\right)^2 \bigg ) \bigg]
    \end{align*}
minimizes the energy distance between $p_t$ and $q$.  
\label{rmk:energy}
\end{remark}
The proof follows from a straightforward application of~\citet[Section 4]{santambrogio2017euclidean}. 
Note that this construction is only possible through distributional dependence in MV-SDE whereas the standard It\^o SDE does not satisfy this drift.
This has particular impact on generative modeling which we will later discuss in our experiment section.

\subsection{Relationship to attention}

Recently, works such as~\citet{sander2022sinkformers} described the relationship between interacting particle systems and the attention structure in the transformer architecture.
Here we briefly describe a motivation for using the proposed architectures in the sense that they describe a similar structure to attention. 

Recall that the attention module is defined by matrices $W_K, W_Q \in \mathbb{R}^{N_W \times d}, W_V \in \mathbb{R}^{N_V \times d}$ and the normalized attention matrix by
$$
\alpha_{i,j} = \frac{N \exp(\langle W_K X^{(i)}, W_Q X^{(j)}\rangle )}{ \sum_{k=1}^{N}\exp(\langle W_K X^{(i)}, W_Q X^{(k)}\rangle ) }.
$$
We focus on the attention matrix since it describes the dependence between particles $X^{(i)}$.

We can rewrite the above equation as an expectation   
$$
\alpha_{i,j} =  \frac{\exp(\langle W_K X^{(i)}, W_Q X^{(j)}\rangle )}{ \mathbb{E}_\nu[\exp(\langle W_K X^{(i)}, W_Q y\rangle )] },
$$
where the expectation is taken with respect to a discrete measure $\nu = \sum_{k=1}^{N} \delta_{X^{(k)}}$, as we do in the IM architecture. 
We can write the numerator as the expectation with an indicator and the denominator as the full expectation, 
$$
\alpha_{i,j} = \frac{ \mathbb{E}[\exp(\langle W_K X^{(i)}, W_Q y\rangle ) \mathds{1}_{y= X^{(j)}} ] }{\mathbb{E}[\exp(\langle W_K X^{(i)}, W_Q y\rangle )]}.
$$
Finally, since we do not assume a particular structure on $\varphi$ in the IM architecture, we can let $\varphi$ be equal to the exponential of the dot product with the transformation by $W_K, W_Q$.
Note that this is applied to particles at each time marginal $t$ rather than for a sequence of particles. 
A sequence of particles would correspond to the case of non-exchangability, which is a direction of future work.

\section{Numerical experiments}
\label{sec:exp}

 For \emph{Q1} we discussed modeling and inferring distributional dependence.  
 We now wish to answer \emph{Q2} and quantify the effect of explicit distributional dependence in machine learning tasks. 
We test the methods on synthetic and real data for time series and generative modeling. 
The main goal is to determine the difference between standard Neural It\^o-SDE and the proposed Neural MV-SDEs under different modeling scenarios. 
In that sense, the baseline we consider is the It\^o-SDE parameterized using an MLP
However, we also consider other deep learning based methods for comparison in a broader context.
We abbreviate the different architectures as the Neural It\^o-SDE (MLP) and Neural MV-SDEs: Empirical Measure (EM), Implicit Measure (IM) and Marginal Law (ML) architectures. These architectures were presented in Section~\ref{sec:architectures} and summarized in~\figurename~\ref{fig:arch}.
Full descriptions of the models, baselines, and datasets are given in the appendix. 
\paragraph{Synthetic data experiments} 
\label{sec:synthetic_data}
We first consider the application of MV-SDEs in physical, biological, social, and financial settings.
These relate to the original development of MV-SDEs and consider how the architecture can be applied in scientific machine learning settings.
We benchmark the proposed methods on 4 canonical MV-SDEs: the Kuramoto model which describes synchronizing oscillators
~\citep{Sonnenschein2013kura}, the mean-field FitzHugh-Nagumo model which characterizes spikes in neuron activations
~\citep{mischler2016kinetic}, the opinion dynamic model on the formation of opinion groups
~\citep{sharrock2021parameter}, and the mean-field atlas model for pricing equity markets~\citep{jourdain2015capital}. 
These models exhibit the non-local behavior that was originally of theoretical interest.
We additionally benchmark the proposed methods on two It\^o-SDEs: an Ornstein–Uhlenbeck (OU) process and a circular motion equation to determine the performance on It\^o-SDEs. 
Finally, to understand the performance on discontinuous paths related to aggregation behavior, we benchmark the proposed methods on an OU process with jumps in Figure~\ref{fig:jump_results}.

Since the true drifts of the synthetic data are known, we directly compare the estimated drifts to the true drifts using mean squared error (MSE). 
The performance on five different datasets with three different levels of added observational noise is presented in Figure~\ref{fig:noise_result}. 
The bottom row of Figure~\ref{fig:noise_result} illustrates an example of the density and the gradient flow at the terminal time for the Kuramoto model.
The proposed mean-field architectures outperform the standard MLP in modeling MV-SDEs; moreover, incorporating explicit distributional depedence does not diminish the performance in estimating It\^o-SDEs.
When modeling processes with jump discontinuities, Figure~\ref{fig:jump_results} highlights the flexibility of the proposed methods, IM, ML, to match such models.
The EM likely does not perform as well due to the high variance of the empirical measure, leading to difficulties in learning.
Additionally, the MLP does not have an explicit decomposition between the MV and It\^o components, resulting in issues when estimating the feedback between the particles inducing jumps.  
\begin{figure}
    \centering
    \begin{subfigure}[t]{0.22\textwidth}
        \includegraphics[width=\textwidth]{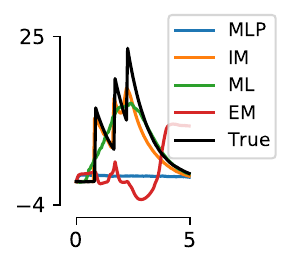}
        \caption{Average paths.}
    \end{subfigure}
    \begin{subfigure}[t]{0.22\textwidth}
        \includegraphics[width=\textwidth]{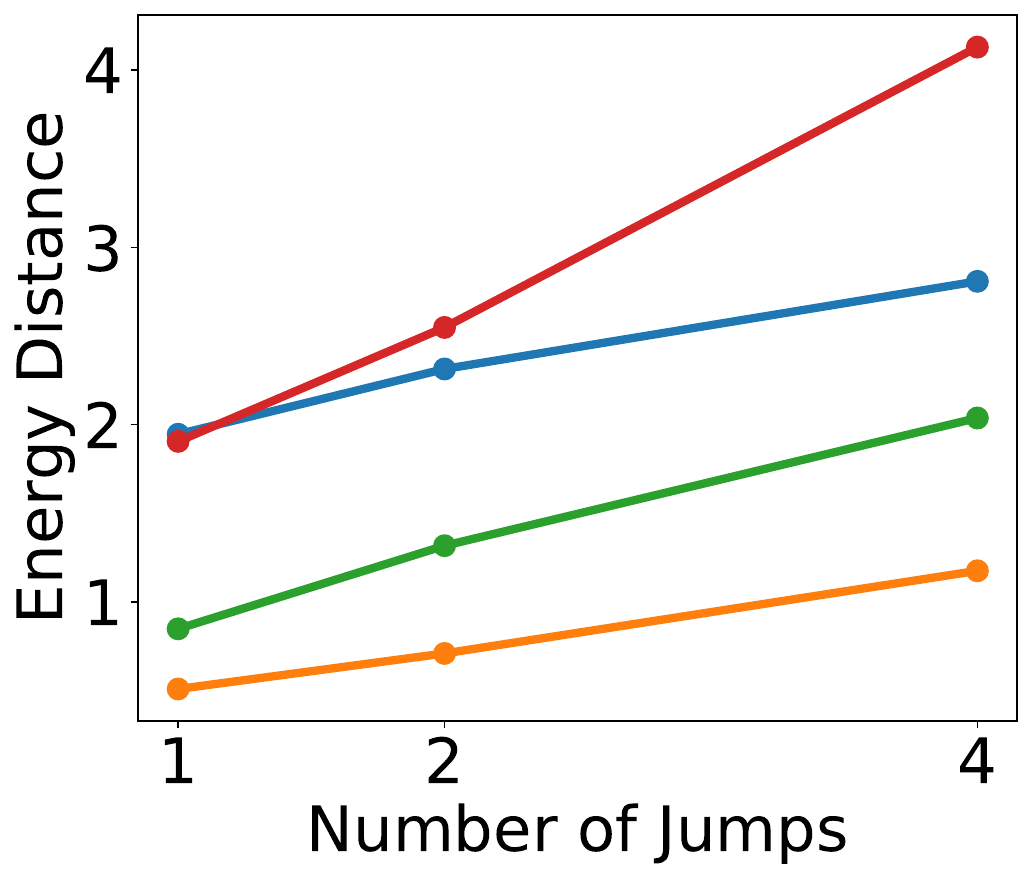}
        \caption{Energy distance.}
        \label{fig:jump_results}
    \end{subfigure}

    \caption{Results for approximating sample paths containing jumps.}
\end{figure}
\paragraph{Real data experiments}
We consider two real examples: crowd trajectory in an open interacting environment, which is related to the Cucker-Smale model~\citep{cucker2007emergent, warren2018collective}
and chemically stimulated movement of organisms (chemotaxis), which can be described using the Keller-Segel model~\citep{tomavsevic2021new, Keller1971KSChemo}. 

\begin{figure}
    \centering
    \includegraphics[width=0.22\textwidth, trim= 20pt 20pt 20pt 20pt]{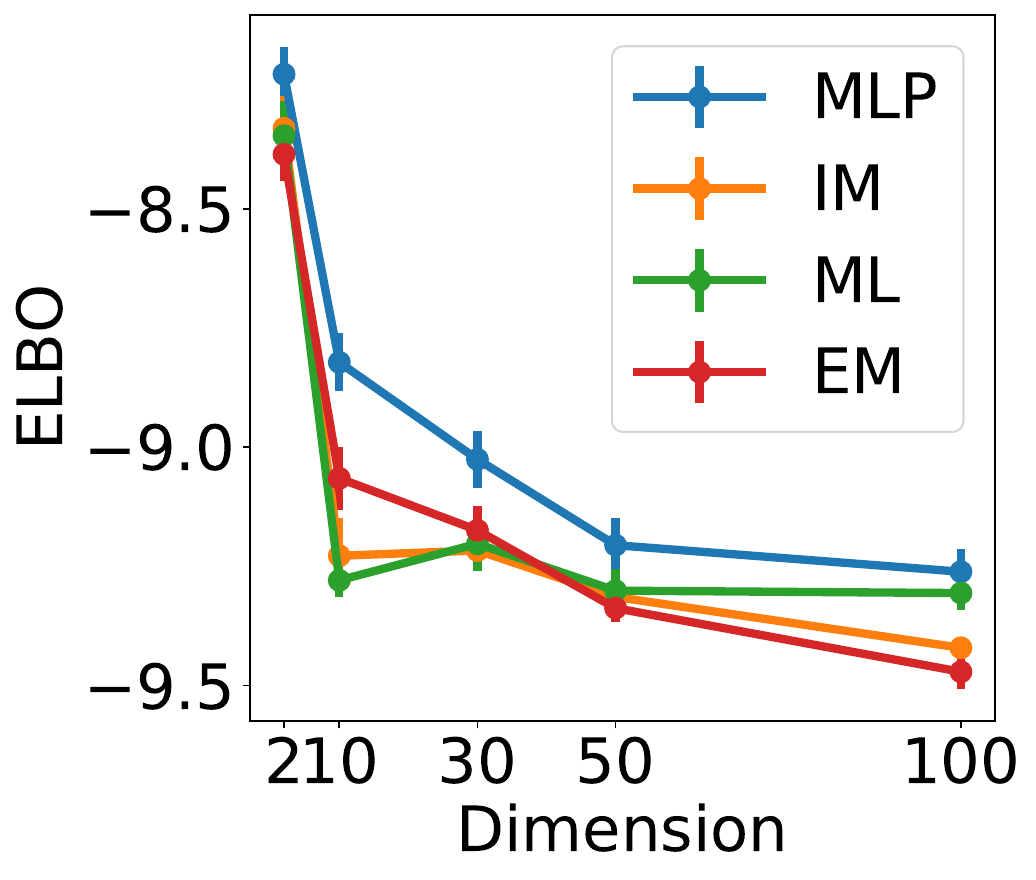}
    \captionof{figure}{ELBO of generated paths from standard Gaussian to eight Gaussian mixture (in increasing dimension) evaluated against OT mapping.}
   \label{fig:eightgauss}
\end{figure}

We evaluate the proposed architectures in these modeling tasks by comparing the goodness-of-fit of generated path samples to the observed path samples, measured in normalized MSE (normalized with sample variance) with respect to the held out data. 
We also benchmark against the DeepAR probabilistic time series forecaster~\citep{salina2020deepAR} with RNN, GRU, LSTM, and Transformer (TR) backbones as baseline models to compare the goodness-of-fit. 
This provides context of the performance within the standard deep learning-based time series forecasting methods.

The performances of different architectures are presented in Table~\ref{tab:realdataTS}.
For the EEG experiments, the proposed architectures generally perform better than the baselines in generating paths within the training time steps, and on par with the DeepAR architectures for forecasting (full results presented in the appendix). 
For the crowd trajectory data, the proposed MV-SDE architectures outperform the MLP, EM and DeepAR architectures for forecasting.
Notably, the $\mathrm{EM}$ architecture exhibits high variance in the crowd trajectory data, indicating the difficulty of relying on the empirical margins to compute expectations. 
For chemotaxis data, the MV-SDE based architectures outperform the DeepAR baselines.
Additional experiments and results are presented in appendix. 
Figures illustrating the sample paths are included in the appendix.

\paragraph{Generative modeling experiments}
We focus on applying the bridge estimator discussed in Section~\ref{sec:est_bb} to map between a Gaussian and a target distribution for the purposes of generative modeling.
This experiment is used to understand the effect of distributional dependence on the quality of samples generated. 
We are interested in studying two aspects: 1) the properties of the learned mapping, and 2) the generated trajectories. 
We first study the properties of the learned mapping using a synthetic eight Gaussian mixture with increasing dimensionality. 
We compare the performance of different architectures by evaluating the ELBO of the sample paths generated by the optimal transport (OT) mapping between the initial distribution and held out target samples.
Then we evaluate the generated trajectories through the energy distance motivated by Remark~\ref{rmk:energy} between generated and held-out data for 5 real data density estimation experiments, since the MV-SDE describes the flow that minimizes the energy distance.  
In addition, we compare to common density estimators of variational autoencoder (VAE)~\citep{kingma2013auto}, Wasserstein generative adversarial network (WGAN)~\citep{gulrajani2017improved}, masked autoregresive flow (MAF)~\citep{papamakarios2017masked} and score-based generative modeling through SDEs, which corresponds to a constrained form of the MLP architecture~\citep{song2020score}.

The MV-SDE architectures not only outperform the It\^o architecture for all dimensions in the eight Gaussian experiment, as shown in Figure~\ref{fig:eightgauss}, but also for the 5 real data density estimation experiments, as shown in Table~\ref{tab:realdataGen}, while outperforming common baselines.  
Sampling is performed using standard Euler-Maruyama, with full details in the appendix. 
This again suggests the MV-SDE provides a more amenable probability flow for generative modeling compared with the It\^o-SDE.

\begin{table*}
    \caption{Time series estimation on held out trajectories. Values in \textbf{bold} and \textit{italic} are best and second best, respectively.}
    \label{tab:realdataTS}
    \centering
    \begin{tabular}{@{}llll@{}}
    \toprule
                  & Crowd Traj             & C.Cres                  & E.Coli            \\  \toprule
    MLP (It\^o)   &  0.068 (0.03)           & 0.096           (0.002) & \textit{0.080} (0.003) \\ 
    IM            &  \textit{0.034} (0.01)  & 0.094           (0.003) & \textbf{0.080}  (0.001) \\
    ML            &  \textbf{0.016} (0.01)  & \textbf{0.093}  (0.002) & 0.084          (0.002) \\
    EM            &  0.091 (0.059)           & \textit{0.093} (0.004)  & 0.086          (0.004) \\ \midrule
    LSTM          &  1.408 (0.92)            & 1.159           (0.234) & 0.585          (0.350) \\
    RNN           &  1.05  (0.54)           & 1.563           (1.070) & 0.773          (0.092) \\
    GRU           &  1.339 (0.61)          & 0.826           (0.289) & 0.568          (0.301) \\
    TR            &  2.732 (0.88)           & 1.503           (0.212) & 1.204 (0.212)           \\\bottomrule
    \end{tabular}
\end{table*}

\begin{table*}

    \centering
    \caption{Density estimation: energy distance between observed samples and generated samples. Values in \textbf{bold} and \textit{italic} are best and second best, respectively.}
    \label{tab:realdataGen}
    \begin{tabular}{@{}llllll@{}} 
    \toprule
                & \textsc{POW} & \textsc{MINI} & \textsc{HEP} & \textsc{GAS} & \textsc{CORT} \\ \toprule
    MLP (It\^o) & 0.34 (0.1)          & 0.67 (0.05)           & 0.54 (0.05)          & 0.41 (0.08)            & 0.74 (0.06) \\ 
    IM       & 0.29 (0.08)          & \textbf{0.40} (0.0)  & 0.41 (0.03)          & \textbf{0.29} (0.08)  & \textbf{0.53} (0.03) \\
    ML & \textbf{0.28} (0.08) & \textit{0.44} (0.03)  & \textit{0.37} (0.03)  & 0.31 (0.06)           & 0.57 (0.03) \\
    EM       & 0.33 (0.1)          & 0.46 (0.04)           & 0.43 (0.05)          & \textit{0.30} (0.03)  & 0.58 (0.037) \\ \midrule
    VAE         & 1.2 (0.02)           & 2.1 (0.15)           & 1.8 (0.03)          & 1.5 (0.02)  & 2.4 (0.2) \\
    WGAN       & 1.2 (0.02)          & 2.1 (0.003)           & 1.8 (0.01)          & 1.3 (0.02)    & 2.2 (0.01)  \\
    MAF         & \textit{0.29} (0.04) & 0.48 (0.01)           & \textbf{0.31} (0.02) & 0.52 (0.03)  & \textit{0.53} (0.03) \\
    Score & 0.30 (0.05)           & 0.50 (0.02)         & 0.32 (0.03)          & 0.56 (0.04)  & 0.58 (0.02)\\
    \bottomrule
    \end{tabular}
\end{table*}

\section{Discussion}
In this paper we discuss an alternative viewpoint of diffusion type models beyond the standard It\^o-SDE parameterization.
In particular, we focus on MV-SDEs and discuss neural representations of a process that depends on the distribution, and ways of making this dependence more explicit. 
We demonstrated the efficacy of the proposed architectures on a number of synthetic and real benchmarks. 
The results suggest that the proposed architectures provide an improvement in certain time series and generative modeling applications, likely due to the more general probability flow that the MV-SDEs induce.

\textbf{Limitations} 
We studied the implicit regularization of the IM architecture under gradient descent, and the extension of the analysis to the other proposed architectures is important to understand the corresponding regularization. 
With regard to computing expectations, using a multilevel scheme~\citep{szpruch2019iterative} could improve accuracy while reducing computational cost. 

\paragraph{Future directions} 
The proposed architectures provide a baseline to extend the work to estimation of alternative processes. 
Heterogeneity amongst the particles is a useful property in many types of systems, e.g. described in~\citet{lacker2022case}. 
Extending the $W_0$ architectures to the case of heterogeneous agents corresponds to introducing depth into the architecture (i.e. having multiple measures $W_0$ to take the expectation with respect to). 
Additionally, solving inverse problems using Wasserstein gradient flows solved using MV-SDEs can be another application of the proposed methods~\citep{crucinio2022solving}.
Interpreting the $W_0$ architecture through the interpolation lens used in~\citet{szpruch2019iterative} could also provide avenues for improvement of the architecture.
Developing optimal estimators for MV-SDE based point processes using the proposed architectures to extend Ito-SDE based point process representations (e.g. in~\citet{hasan2023inference}) could be a useful direction for extension when observations are only given as arrival times of events. 
Finally, establishing convergence rates for the architectures such as the $W_0$ or $X_t$ architectures would possibly be a direction for further analysis on the proposed algorithms.

\subsubsection*{Acknowledgements}
This work was supported in part by the Air Force Office of Scientific Research under award number FA9550-20-1-0397.
AH was partially supported by an NSF Graduate Research Fellowship.

\clearpage

\bibliography{refs}
\bibliographystyle{plainnat}

\appendix
\onecolumn
\aistatstitle{Neural McKean-Vlasov Processes: \\
Supplementary Materials}
\section{Proofs}

In the main text, we briefly discussed some theoretical advantages in terms of the flexibility in the time marginal distributions that MV-SDEs can represent versus It\^o-SDEs, such as non-local dynamics and jumps in the sample paths. For more background and properties, we refer to the notes by~\citet{lacker2018notes} and the book by~\citet{carmona2018probabilistic}.

In this section, we begin by considering the theoretical advantages of the proposed architectures and estimators. Specific to the architectures, we develop the implicit measure architecture, with a comparison to attention. In addition, we study the implicit regularization of explicit distributional dependence, with a comparison to optimal transport. Specific to the estimators, we develop the compatibility criterion for the modeled density to be consistent with the flow of the modeled SDE, and discuss a similar interpretation for the interpolation approach of the Brownian bridge estimator.

\subsection{Development of implicit measure architecture}

The implicit measure (IM) architecture is motivated by the fact that given a drift $b$ that is Lipschitz continuous, by the universal approximation theorem, a two-layer multi-layer perceptron (MLP) can approximate $b$ to arbitrary precision~\citep{hornik1991approximation}. We first show that the drift of a MV-SDE may be represented by a MLP then describe the IM architecture where the distributional dependence is made more apparent.

\begin{proof}
    Consider a McKean-Vlasov process where the drift $b$ is factorized into a linear form
    $$
    b(X_t,p_t,t) =  f(X_t, t) + \mathbb{E}_{y_t \sim p_t}[\varphi(X_t-y_t)]
    $$
    and assume that $f(\cdot;\theta)$ and $\varphi(\cdot;\theta)$ are well approximated by MLPs following the universal approximation theorem. 
    It remains to show that $\mathbb{E}_{y_t\sim p_t}[\varphi(X_t-y_t)]$ can be well approximated by an MLP. We will begin by presenting the proof for the case where the law is stationary, then perform a change of measure to extend to the case where the law is non-stationary.

    Recall that a MLP can be written in terms of an expectation as
    \begin{align*}
    \mlp^{W,b}(x) &= \int \sigma \left ( Wx + b \right ) \mathrm{d}\nu\left(W,b\right) \\
    &= \mathbb{E}[\sigma(Wx+b)]
    \end{align*}
    where the expectation is taken over $\nu( \cdot )$, a measure over the space of parameters $W,\, b$, and $\sigma$ is an activation function. 
    By our original argument that $\varphi$ is well approximated by a MLP, we can let that represent the activation function.     Next, set $\nu(W)=\delta_{I_d}$ and $\nu(b) = \mathrm{Law}(-X_t)$. Since we assumed $X_t$ is stationary, $\mathrm{Law}(X_t) = \mathrm{Law}(X_\star)$ for all $t$.
    We now have our approximation as 
    \begin{align*}
    \mlp^{W,b}(x) &= \int \varphi \left ( x - b \right ) p_t(b)\mathrm{d}b \\
    &= \mathbb{E}_{y\sim p_t}[\varphi(x-y)].
    \end{align*}

    \paragraph{Non-stationary law}
    Next we consider the case where the law of $X_t$ is not the same for all $t$.
    For this argument, we will consider the change of measure that maps $p_t$ to $q$.
    Since we are assuming that the diffusion is constant, all measures $p_t$ are absolutely continuous with respect to each other.
    We additionally assume that Novikov's condition is satisfied. 
    Following Girsanov's theorem, we can write the expectation in terms of this changed measure by introducing the factor $\frac{\rm{d}p_t}{\rm{d}q}$
    $$
    \mathbb{E}_{y\sim p_t}[\varphi(x-y)] = \mathbb{E}_{y\sim q}\left[\varphi( x - y) \frac{\mathrm{d} p_t}{\mathrm{d} q} \right].
    $$
    Under this formulation $q$ is the learned measure and $p_t$ is the measure at each time point $t$.
    Assuming that the function $\varphi(\cdot;\theta) \frac{ \mathrm{d}p_t}{\mathrm{d} q}$ can be learned for all $t$ as another MLP $\tilde{\varphi}\left ( \cdot , t;\theta \right )$, we conclude the proof. 
    A similar idea was explored in~\citet{du2021empirical} where the authors attempt to compute a stationary measure as a change of measure of particle samples. 
    
\end{proof}

    Following a similar notation to the MLP proof, we only change the base measure such that it is given by the mean-field layer.
    A similar change of measure argument is then applied to complete the development of the IM architecture. 

The proposed neural architectures differ from existing methods that consider the empirical measure, since we consider parameters to describe the measure at different time points. 
The proposed neural architectures also differ from existing methods that describe Ito-SDEs since we consider a more explicit parameterization of distributional dependence and a more general interaction function, $\varphi(x,y)$ instead of $\varphi(x-y)$.

\subsection{Implicit regularization of explicit distributional dependence}
\label{sec:bias_proof}

\begin{proof}

Consider a McKean-Vlasov process governed by
$$
\mathrm{d}X_t = \left\{ f(X_t, t) + \mathbb{E}_{y_t \sim p_t}[\varphi(X_t, y_t)]  \right\} \mathrm{d} t + \mathrm{d}W_t.
$$

Our goal is to understand the implicit regularization of the IM architecture where the expectation is approximated by a discrete measure $\nu=\frac{1}{w}\sum_{k=1}^w \delta_{\theta_k}$ and $\theta_k$ corresponds to the $k$th row of a $w\times d$ weight matrix $\theta$. We show that the path preferred by gradient descent is the one that minimizes $ \mathbb{E}_{y_t \sim \nu} \left[\varphi(X_t, y_t) \right] $, i.e. the solution with least influence from other particles. In addition, when $\varphi$ can be decomposed as a norm, this amounts to finding the weights with smallest norm. For ease of notation, we will begin by presenting the proof for one time step and in 1-dimension, i.e. $d=1$.

Following the blueprint given by~\citet{belabbas2020implicit}, we wish to study the implicit bias of the weight matrix $\theta$ by understanding the compatibility between two optimization problems, the \emph{training} problem given by the loss:
\begin{equation}\label{eq:loss_ir}
\min_\theta\mathcal{L}(\theta, X) = \sum_{i=1}^N\frac1{2\Delta_t^2} \left( (X_{t+\Delta_t}^{(i)} - X_t^{(i)}) - \left( f(X_t^{(i)},t)  + \frac1w \sum_{k=1}^w \varphi(X_t^{(i)}, \theta_k) \right) \Delta_t \right)^2
\end{equation}
for observations $\left\{X_t^{(i)}, X_{t+1}^{(i)}\right\}_{i=1\ldots N}$ and the \emph{regularization} problem given by
$$
\min_\theta K(\theta,X)\quad\textrm{s.t.}\quad\mathcal{L}(\theta,X)=0
$$
for some function $K$ that satisfies the PDE:
\begin{equation}\label{eq:pde_k}
\frac{\partial^2 K}{\partial \theta^2} g(\theta,X) + \sum_{i=1}^N \lambda_i \frac{\partial^2 \mathcal{L}}{\partial\theta^2}(\theta, X^{(i)}) g(\theta,X) = 0
\end{equation}
where $g$ denotes the dynamics of gradient descent given as
$$
g(\theta,X)=\dot{\theta}=\sum_{i=1}^N\frac{\partial \mathcal{L}}{\partial \theta}(\theta,X^{(i)}).
$$
Following~\citet{belabbas2020implicit}, the PDE\eqref{eq:pde_k} has a simple interpretation: the Hessian of $K$, acting on $g$, is a linear combination of the Hessians of $\mathcal{L}$ at datapoints $X^{(i)}$, acting on $g$. The next step is to find the function $K$.

We compute the first derivative
$$
\partial_{\theta_j} \mathcal{L}^{(i)} = \left(\frac{-1}{\Delta_t w}\left (( X_{t+\Delta_t}^{(i)} - X_t^{(i)}) - \left(f^{(i)} + \frac1w \sum_{k=1}^w \varphi(X_t^{(i)}, \theta_k) \right)\Delta_t \right )  \partial_{\theta_j} \varphi(X_t^{(i)}, \theta_j)   \right).
$$
Then the second derivative as
\begin{align*}
\partial_{\theta_j, \theta_j} \mathcal{L}^{(i)} = &\biggl (\frac1{w^2}(\partial_{\theta_j} \varphi(X_t^{(i)},\theta_j))^2 \\&- \frac{1}{\Delta_t w}\left ( (X_{t+\Delta_t}^{(i)} - X_t^{(i)}) - \left(f^{(i)} +  \frac{1}{w}\sum_{k=1}^w \varphi(X_t^{(i)}, \theta_k) \right) \Delta_t\right )  \partial_{\theta_j, \theta_j} \varphi(X_t^{(i)}, \theta_j)  \biggr).
\end{align*}
with the off-diagonal second derivative as
$$
\partial_{\theta_k, \theta_j} \mathcal{L}^{(i)} = \frac{1}{w^2} \partial_{\theta_i}\varphi(X_t^{(i)}, \theta_k) \partial_{\theta_j}\varphi(X_t^{(i)}, \theta_j).
$$
The terms with coefficient $\frac{1}{w^2}$ will have coefficient $\frac{1}{w^3}$ when multiplied by the first partial derivative in $g$. 
Taking $w = \mathcal{O}(1/\Delta_t)$, these terms are negligible. With the choice of 
$$
\lambda_i =  \Delta_t \left( (X_{t+\Delta_t}^{(i)} - X_t^{(i)}) - \left(f^{(i)} + \frac1w \sum_{k=1}^w \varphi(X_t^{(i)}, \theta_k) \right)\Delta_t \right )^{-1}
$$
we obtain the PDE
$$
\frac{\partial^2 K}{\partial \theta^2} - \sum_{i=1}^N\frac{1}{w} \sum_{k=1}^w\partial_{\theta_k, \theta_k} \varphi(X_t^{(i)}, \theta_k) = 0.
$$
This suggests that the regularization problem that we are solving, repeating for $T$ time steps, is
\begin{equation}
\min_\theta K(\theta,X) = \sum_{t=1}^T\sum_{i=1}^N\frac1w \sum_{k=1}^{w}\varphi(X_t^{(i)}, \theta_k)\quad\textrm{s.t.}\quad\mathcal{L}(\theta,X)=0.
\label{eq:dual}
\end{equation}
In the context of the MV-SDE, the mean-field system approximated is the one that has the least influence from the other particles. 

\paragraph{$d$-dimensions.}
Now consider the case with $\theta_k$ as vectors.
The notation becomes more complex as the partial derivatives now form tensors. 
However, since the diffusion is assumed to be constant and diagonal, we can give a brief analysis similar to the 1-dimensional case. 
The loss function is now 
$$
\mathcal{L}(\theta, X) = \sum_{t=1}^T\sum_{i=1}^N\frac1{2\Delta_t^2} \sum_{j=1}^d \left( (X^{(i)}_{t+\Delta_t} - X^{(i)}_t) - \left( f(X_t^{(i)},t)  + \frac1w \sum_{k=1}^w \varphi(X_t^{(i)}, \theta_k) \right) \Delta_t \right)^2_j.
$$

The regularized problem has a similar form of
\begin{equation*}
\min_\theta K(\theta,X) = \sum_{t=1}^T\sum_{i=1}^{N}\frac1w \sum_{k=1}^{w}\sum_{j=1}^d  \varphi(X_t^{(i)} , \theta_k)_j\quad\textrm{s.t.}\quad\mathcal{L}(\theta,X)=0.
\end{equation*}
\end{proof}

\subsubsection{Comparison to optimal transport}
Now consider the case $f=0$ and recall that the transition density satisfies the PDE
\begin{equation}
\partial_t p(x,t) = -\nabla \cdot \left ( \int_\Omega \varphi (x, y) p(y, t) \mathrm{d}y p(x,t) \right ) + \frac{\sigma^2}{2} \nabla^2 p(x,t)
\label{eq:flow}
\end{equation}
such that $p(x,0) = p_0(x)$ and $p(x,T) = p_T(x)$.
Suppose that $\varphi$ can be represented as a norm $\|g(x,y)\|^2$ and replace the drift with the one given by the implicit bias, the PDE then becomes
\begin{align*}
\partial_t p(x,t) &= -\nabla \cdot \left ( \min_{\nu} \left( \int_\Omega \| g(x, y) \|^2 \nu(y, t)  \mathrm{d}y \right) p(x,t) \right ) + \frac{\sigma^2}{2}\nabla^2 p(x,t) \\
&= - \nabla \cdot \left ( \min_{\nu} \mathbb{E}_{y\sim\nu} \left [ \| g(x, y) \|^2 \right ]  p(x,t) \right ) + \frac{\sigma^2}{2}\nabla^2 p(x,t)\\
&= - \nabla \cdot \left ( \min_{g} g(x, t)  p(x,t) \right ) + \frac{\sigma^2}{2}\nabla^2 p(x,t)
\end{align*}
where the last step can be seen as a parameterization of the function $g$ by the measure $\nu$.

We see some similarities to the Benamou-Brenier form of the Wasserstein-2 distance, where the optimization problem is given by
\begin{equation}
W_2(\rho, \mu) = \min_{g} \int_0^T \mathbb{E}_{X_t \sim p(x,t)} \left[\|g(X_t,t)\|^2\right] \mathrm{d}t
\label{eq:bb}
\end{equation}
subject to 
\begin{equation}
\partial_t p = - \nabla \cdot \left ( g(x,t) p(x,t)\right ), \quad p_0(x) = \rho, \: p_T(x) = \mu.
\label{eq:bb-transport}
\end{equation}
Compare~\eqref{eq:dual} to~\eqref{eq:bb} where we have the same objective.
In addition, note that the probability flow~\eqref{eq:flow} satisfies the transport equation~\eqref{eq:bb-transport} in the limit as $\sigma \to 0$.
This lends to an interpretation that, under certain choices of $\varphi$, the problem relates to the entropy regularized optimal transport problem under the $W_2$ cost. 
Notably, this comes as a result of the implicit bias introduced by the neural network gradient optimization scheme and is not a separate term that needs to be added. 

\label{sec:add_formula}


\subsection{Compatibility criterion in inferring explicit distributional dependence}

\subsubsection{Feynman-Kac for the Kolmogorov backward equation}
\label{sec:kbk}

The Kolmogorov backward and forward equations are PDEs that describe the time evolution of the marginal density of the associated SDE. The Kolmogorov backward equation describes the evolution of the density when given a known terminal condition. Its adjoint, the Kolmogorov forward equation, establishes an initial condition and provides the density at some future time. In this section, we focus on regularizing the modeled density to be consistent with the flow of the modeled SDE using the Kolmogorov backward equation. In Section~\ref{sec:fk}, we derive a likelihood and perform additional generative modeling experiments based on a linearization of the Kolmogorov forward equation, also known as the Fokker-Planck equation.

For the modeled density to be consistent with the flow of the modeled SDE, it has to satisfy the Kolmogorov backward equation defined as
\begin{equation}\label{eq:kbe}
-\partial_t p_t = b(\cdot) \nabla p_t + \frac{\sigma^2}{2} \nabla^2 p_t.
\end{equation}
A solution to the above equation is given by the Feynman-Kac formula as an expectation of trajectories at terminal time, i.e.
\begin{equation}\label{eq:kbe_fk}
p_t(x) = \mathbb{E}\left[ p_T(X_T) \mid X_t = x \right]
\end{equation}
where $p_T(\cdot),\, t<T$ is the given terminal condition and $X_s$ satisfies the SDE $\mathrm{d}X_s = b(\cdot)\mathrm{d}s + \sigma \mathrm{d}W_s.$

Following~\eqref{eq:kbe_fk}, we evolve $X_s$ from $X_t=x$ to $X_T$, then penalize the difference between $p_t(x)$ and $\mathbb{E}\left[ p_T(X_T) \mid X_t = x \right]$. The estimation algorithm with this compatibility criterion on the marginal density is detailed in Algorithm~\ref{alg:cc}.

\subsubsection{Feynman-Kac analysis of the Brownian bridge estimator}\label{sec:proof_bb}
Consider the bridge estimator
$$
\mathcal{L}_{BB} = P(\{Z_{t_{j+1}} = X_{t_{j+1}}\} \mid Z_{t_{j}} = X_{t_{j}})=\mathbb{E}_{\mathbb{Q}} \left [\mathds{1}\{ Z_{t_{j+1}} = X_{t_{j+1}}\} \mid Z_{t_{j}} = X_{t_{j}}\right ]
$$
where the expectation is taken over Brownian paths $Z_t$ under the Wiener measure $\mathbb{Q}$. 
This computes the probability that a Brownian motion $Z_t$, conditioned to be equal to $X_{t_{j}}$ at $t_j$, is equal to $X_{t_{j+1}}$ at $t_{j+1}$. This can be thought of using the Kolmogorov backward equation and Feynman-Kac formula from the previous section. Applying a change of measure using Girsanov's theorem to a drifted Brownian motion, we arrive at the estimator described in the main text
$$
\mathcal{L}_{BB}(\theta) = \mathbb{E}_{\mathbb{Q}} \left [ \mathds{1}\{ Z_{t_{j+1}} = X_{t_{j+1}}\} \exp \left( \int_{t_{j}}^{t_{j+1}} b(\,\cdot\, ; \theta) \mathrm{d}Z_t - \frac12 \int_{t_{j}}^{t_{j+1}} b(\,\cdot\, ; \theta)^2 \mathrm{d}t \right) \mid Z_{t_{j}} = X_{t_{j}}\right ].
$$
The indicator function, which acts as the boundary condition for the Kolmogorov backward equation, restricts the paths of $\mathbb{Q}$ to those that are Brownian bridges between $X_{t_j}$ and $X_{t_{j+1}}$. The change of measure via Girsanov's provides the mechanism for inferring the optimal drift for the observed data.

The experiments then provide a way of evaluating whether including distributional properties in the drift (i.e. \emph{nonlinear} Kolmogorov backward equation with $b(X_t,p_t,t;\theta)$) results in better probabilities than without (i.e. linear Kolmogorov backward equation with $b(X_t,t;\theta)$). 

\subsection{Gradient flow minimizing the energy distance}
The energy distance is defined as
$$
E(p,q) = 2 \int \int \| X - Y\| \mathrm{d}q(X) \mathrm{d}p(Y) - \int \int \| X - X'\| \mathrm{d}q(X) \mathrm{d}q(X') - \| Y - Y'\| \mathrm{d}p(Y) \mathrm{d}p(Y')
$$
and by linearity of expectations, we can rewrite as
$$
E(p,q) = 2 \int \int \left (\underbrace{\| X - X'\| \frac{\mathrm{d}p(Y)}{\mathrm{d}q(X')}  - \| X - X'\| - \| X - X'\|\left( \frac{\mathrm{d}p(Y)}{\mathrm{d}q(X')}\right )^2}_{\varphi}\right )\mathrm{d}q(X) \mathrm{d}q(X').
$$
Our goal is to find a gradient flow that minimizes this distance. 
Following~\citet{santambrogio2017euclidean}, we can use the gradient of this function as the drift to promote attraction to the target density.
We will use the form within the parenthesis as $\varphi$ for representing the aggregation potential, $\mathcal{W}(p) = \int \int \varphi(X, X') \mathrm{d}p(X) \mathrm{d}p(X')$.
Additionally, we include the Radon-Nikodym derivative for this between the target density and the current time density to define the energy distance.
Applying the gradient to $\varphi$, we achieve the desired result.
While we do not explicitly impose this, we use this idea to motivate our results on generative modeling.

\clearpage

\section{Algorithms}

To supplement the algorithmic contributions in the main paper we detail the inference procedure for the regular time observations in Algorithm~\ref{alg:girsanov} and the irregular time observations with Brownian bridges in Algorithm~\ref{alg:bridge}. 
We then detail the inference procedure with regularization of the marginal law using a compatibility criterion of the PDE with the associated SDE in Algorithm~\ref{alg:cc}. 
Finally, we describe a sampling procedure in Algorithm~\ref{alg:sampling}. 
The code is attached in the supplementary material and will be posted online.

\begin{algorithm}[h]
\caption{Maximum Likelihood Estimation (MLE) with Girsanov's Theorem}
\label{alg:girsanov}
\begin{algorithmic}
\STATE {Input: observed trajectories $\left\{\{X_{t_j}\}_{j=1\ldots K}^{(i)}\right\}_{ i=1\ldots N}$.}
\STATE {Initialize: neural drift $b(\cdot;\theta)$.}
\FOR{$i$ in mini-batch}
\FOR{$j$ in $1...K-1$}
     \STATE Compute $\Delta X_{t_j}^{(i)} = X_{t_{j+1}}^{(i)} - X_{t_{j}}^{(i)}$.
     \STATE Compute discretized approximation to $\log$ of exponential martingale: \\
     $\mathcal{L}(\theta) := b(X_{t_j}^{(i)},p_{t_j},t_j;\theta) \Delta X_{t_j}^{(i)}- \frac{1}{2} b(X_{t_j}^{(i)},p_{t_j},t_j;\theta)^2 (t_{j+1} - t_{j}).$
     \STATE Maximize $\mathcal{L}(\theta)$ using gradient based optimizer.
\ENDFOR
\ENDFOR
\end{algorithmic}
\end{algorithm}

In the computation of the mean-field component of  $b(X_{t_j}^{(i)},p_{t_j},t_j;\theta)$, we do the following:
\begin{itemize}
    \item EM architecture: $\frac{1}{n}\sum_{k=1}^n\varphi(X_{t_j}^{(i)},X_{t_j}^{(k)};\theta)$.
    \item IM architecture: $\frac{1}{n}\sum_{k=1}^{n}\varphi(X_{t_j}^{(i)},W_{0}^{(k)},t_j;\theta)$\\ as a neural network with an additional layer and additional conditioning on $t_j$.
    \item ML architecture: $\frac{1}{n}\sum_{k=1}^{n}\varphi(X_{t_j}^{(i)},\hat{X}_{t_j}^{(k)};\theta)$ \\
    where we compute the expectation with samples $\{\hat{X}_{t_j}^{(k)}\}_{k=1}^{n}$ from $\hat{p}(\cdot,t_j;\phi)$,\\ a generative network with additional conditioning on $t_j$.
\end{itemize}

Additional details on the parameterization of the neural architectures are in Section~\ref{sec:exp_params}.

In the case of irregular time observations, for each trajectory, we first sample Brownian bridges between observations, then use the sampled Brownian bridges as regular time observations. 
In this case, the estimation procedure aims to fit the observations while penalizing deviations from the Brownian bridge paths in regions without observations. 
The Brownian bridge approach also has the interpretation of the shortest distance interpolator that exactly fits the margins.
Using a Brownian bridge path construction reduces the variance of the estimator.

\begin{algorithm}[h]
\caption{MLE with Girsanov and Brownian Bridges}
\label{alg:bridge}
\begin{algorithmic}
\STATE {Input: observed trajectories $\left\{\{X_{t_k}\}_{k=1\ldots K}^{(i)}\right\}_{ i=1\ldots N}$.}
\STATE {Initialize: neural drift $b(\cdot;\theta)$.}
\FOR{$i$ in mini-batch}
\FOR{$k$ in $1...K-1$}
     \STATE Sample Brownian bridge $\{Z_{t_{j}}\}^{(i)}_{j=1...J}$ between $X_{t_{k}}^{(i)}$ and $X_{t_{k+1}}^{(i)}$.
     \FOR{$j$ in $1...J-1$}
     \STATE Compute $\Delta Z_{t_j}^{(i)} = Z_{t_{j+1}}^{(i)} - Z_{t_{j}}^{(i)}$.
     \STATE Compute discretized approximation to $\log$ of exponential martingale: \\
     $\mathcal{L}(\theta) := b(Z_{t_j}^{(i)},p_{t_j},t_j;\theta) \Delta Z_{t_j}^{(i)}- \frac{1}{2} b(Z_{t_j}^{(i)},p_{t_j},t_j;\theta)^2 (t_{j+1} - t_{j}).$
     \STATE Maximize $\mathcal{L}(\theta)$ using gradient based optimizer.
     \ENDFOR
\ENDFOR
\ENDFOR
\end{algorithmic}
\end{algorithm}

We next detail the estimation procedure with regularization of the marginal law using the correspondence between the PDE and its associated SDE via the nonlinear Kolmogorov backwards equation~\citep{buckdahn2017mvpde}.

\begin{algorithm}[h]
\caption{MLE with Girsanov and Regularization of Explicit Marginal Law $\hat{p}_t$}
\label{alg:cc}
\begin{algorithmic}
\STATE {Input: observed trajectories $\left\{\{X_{t_j}\}_{j=1\ldots J}^{(i)}\right\}_{ i=1\ldots N}$.}
\STATE {Initialize: neural drift $b(\cdot;\theta,\phi)$, including explicit marginal law $\hat{p}(\cdot;\phi).$}
\FOR{$i$ in mini-batch}
\FOR{$j$ in $1 ... J-1$}
     \STATE Compute $\Delta X_{t_j}^{(i)} = X_{t_{j+1}}^{(i)} - X_{t_{j}}^{(i)}$.
     \STATE Compute discretized approximation to $\log$ of exponential martingale: \\
     $\text{ELBO} := b(X_{t_j}^{(i)},p_{t_j},t_j;\theta,\phi) \Delta X_{t_j}^{(i)}- \frac{1}{2} b(X_{t_j}^{(i)},p_{t_j},t_j;\theta,\phi)^2 (t_{j+1} - t_{j}).$
     \STATE Sample $\{\{Z_{t_{j+1}}|z=X_{t_j}^{(i)}\}^{(k)}\}_{k=1...K}$ following the dynamics of the ML architecture.
     \STATE Compute the expected log-likelihood $\mathbb{E}[\log \hat{p}_{t_{j+1}}(Z_{t_{j+1}};\phi)]=\frac{1}{K}\log \hat{p}_{t_{j+1}}(Z_{t_{j+1}}^{(k)};\phi)$.
     \STATE Compute compatibility criterion $\text{CC} := (\log \hat{p}_{t_j}(X_{t_j}^{(i)};\phi) - \mathbb{E}[\log \hat{p}_{t_{j+1}}(Z_{t_{j+1}};\phi)])^2$
    \STATE Compute total loss
     $\mathcal{L}(\theta) := \text{ELBO} + \text{CC}$.
     \STATE Maximize $\mathcal{L}(\theta)$ using gradient based optimizer.
\ENDFOR
\ENDFOR
\end{algorithmic}
\end{algorithm}

We finally describe a sampling algorithm.

\begin{algorithm}[h]
\caption{Sampling Trajectories with Euler-Maruyama Scheme }
\label{alg:sampling}
\begin{algorithmic}
\STATE {Initialize: time grid $\{t_j\}_{j=1...K}$.}
\STATE {Initialize: initial observations $\{X_0^{(i)}\}_{i=1...N} \sim p_0$.}
\FOR{$j$ in $1...K-1$}
    \STATE Compute $\Delta t_j=t_{j+1}-t_j$.
    \FOR{$i$ in $1...N$}
    \STATE Sample $\Delta W_{t_j}^{(i)}\sim_{iid}\mathcal{N}(0,\Delta t_j)$
    \STATE Compute $X_{t_{j+1}}^{(i)} = b(X_{t_j}^{(i)},p_{t_j},t_j;\theta)\Delta t_j + \sigma \mathrm{d}W_{t_j}^{(i)}$.
    \ENDFOR
\ENDFOR
\end{algorithmic}
\end{algorithm}

\clearpage

\section{Experimental details}

In this section, we detail the evaluation metrics, datasets, hyperparameter settings, and provide additional experiments to supplement the results in the main paper.

\subsection{Evaluation metrics}

\subsubsection{Continuous ranked probability score}
Following~\citet{Gneiting2007CRPS}, the continuous ranked probability score (CRPS) is given by
\begin{align*}
    \rm{CRPS}(F,x) &= \int_{-\infty}^{\infty} \left[F(y) - \mathds{1}(y \geq x)\right]^2 \mathrm{d}y.
\end{align*}
The CRPS evaluates the modeled distribution against a single observation by comparing the cumulative distribution function (CDF) of the modeled distribution $F$ to a step function placed at the observation $x$.

\subsubsection{Energy distance}
\label{sec:energy_dist}
The squared energy distance between two distributions $P_0$ and $P$ is defined as 
$$
d^2(P_0,P) := 2 \, \mathbb{E}_{X \sim P_0, Y \sim P}[ \| X - Y \|] - \mathbb{E}_{X \sim P_0, X' \sim P_0}[ \| X - X' \|] 
- \mathbb{E}_{Y \sim P, Y' \sim P}[ \| Y - Y' \|]
$$
where we compute the expectations empirically. 

\subsubsection{KS-distance}
The KS-distance between two 1-d empirical cumulative distribution functions (ECDF) is defined as:
$$
\text{KS-distance} = \sup_x|F_1(x) - F_2(x)|
$$
where $F_1, F_2$ are two ECDFs. This metric is only used in the appendix for additional 1-dimensional experiment results. 


\subsection{Datasets}
Here we describe the datasets in more detail and provide exact statements on the simulation parameters.

\subsubsection{Synthetic time series data}

\paragraph{Kuramoto model.} 
The Kuramoto model which describes synchronizing oscillators takes the form $$
    \mathrm{d}X_t^{(i)} = \left[ h^{(i)} + \frac{K}{N} \sum_{j=1}^N \sin\left(y_t^{(j)} - X_t^{(i)}\right) \right]\mathrm{d}t + \sigma \mathrm{d}W_t^{(i)},$$ where movements of $N$ particles are governed by a linearly factored drift that includes some function $h^{(i)}$ and a mean-field term that couples the particles. 
    We simulate 2-dimensional trajectories with $X_t^{(i)}=[X_{1t}^{(i)},X_{2t}^{(i)}]\in\mathbb{R}^2$, $h^{(i)} = \left[\sin(X_{1t}^{(i)}), \sin(X_{2t}^{(i)}) \right]$,  $K=2$, $N$=20, and $\sigma=1$. 

\paragraph{Fitzhugh-Nagumo model.} 
The FitzHugh-Nagumo model is a set of equations that models spikes in neuron activations as membrane voltage spikes $X_{1t}$, driven by external stimulus $I_{\mathrm{ext}}$, and diminishing over time $X_{2t}$. It takes the form 
    \begin{align*}
    \mathrm{d}X_{1t} &= \left(aX_{1t}\left(X_{1t} - \lambda\right)\left(1-X_{1t}\right)-X_{2t} + I_{\mathrm{ext}} \right) \mathrm{d}t + \mathbb{E}\left[X_{1t} - y_{1t}\right] \mathrm{d} t + \sigma \mathrm{d}W_t, \\
    \mathrm{d}X_{2t} &= \left(-bX_{2t}+cX_{1t} + d \right)\mathrm{d}t,
    \end{align*}
    We chose $a=0.2, b = 0.8, c=1, d=0.7, \lambda=0.4, I_{\mathrm{ext}} = 0.1\sin(10t)$, and $\sigma=0.3$. The expectation is approximated with $N=20$ particles. 

\paragraph{Opinion dynamic model.}
The opinion dynamic model simulates the opinion formation process through an equation with the form $$dX_t=\mathbb{E}\left[\psi_\theta(||X_t - y_t||)(X_t - y_t)\right] + \sigma \mathrm{d}W_t,$$ where $\psi_\theta(r) = \theta_1 \exp (-\frac{0.01}{1-(r-\theta_2)^2})$. We simulate 2-dimensional trajectories with $\theta_1 = 1, \theta_2 = 2.5$.

\paragraph{Mean-field atlas model.} 
The mean-field atlas model for pricing equity markets takes the form $$\mathrm{d}X_t = \gamma\left ( \int \mathds{1}_{\left\{X_t - y_t > 0\right\}} \mathrm{d}p_t(y)\right) \mathrm{d}t + \sigma \mathrm{d}W_t,$$ 
where the drift $\gamma(\cdot)$ depends on the rank of the particle at each time. 
Let $u = \int \mathds{1}_{\left\{X_t - y_t > 0\right\}} \mathrm{d}p_t(y) \mathrm{d}t$, we define $\gamma = 1 - u\exp(2u)$.

\paragraph{It\^o diffusion -  Ornstein-Uhlenbeck.}
We simulated a 2-dimensional Ornstein-Uhlenbeck (OU) process with drifts $\left [-3X_{1t}, -2X_{2t}\right ].$

\paragraph{It\^o diffusion -  circle.}
We simulated a 2-dimensional SDE with circular evolution given by drifts $\left [-X_{1t} - 2X_{2t}, -X_{2t}+ 2X_{1t}\right ]$.

\paragraph{Jump diffusions.}
We simulated a 2-dimensional OU process with drifts $\left [-X_{1t}, -X_{2t}\right ]$ and additional 1, 2, or 4 jumps sampled uniformly in time with jump size distributed as $\exp(\mathrm{Uniform(2,3)})$.

All models are two-dimensional except the mean-field atlas model that is one-dimensional.

We first simulated samples using the Euler-Maruyama method on a fine grid $\Delta t$, i.e. $X_{t+\Delta t} = X_t + b(X_t, p_t, t)\Delta t + \sigma\Delta W$ with $\Delta W \sim \mathcal{N}(0, \Delta_t)$ and $t\in[0,T]$. For irregular time samples, a batch of observation times are then sampled according to an exponential distribution with rate $\lambda = T/N^\prime$, where $N^\prime$ is the number of irregular time samples.
The sampled timestamps are then matched to the closest times in the discretized time sequence used in sample generation. 
Only the matched timestamps $t^\prime$, the initial condition $X_0$, and the terminal condition $X_T$ are used in training. For evaluation, we consider the full trajectories. Specific choices of $\sigma, T, \Delta t, N$, and $N^\prime$ are provided in Table \ref{tab:synthetic_data}.
To realistically simulate real-world parameter estimation, ``observation noise" in the form of Gaussian with standard deviations $\in \{0.1, 0.5, 1\}$ is added to the sampled data. 

\begin{table}[h]
\small
    \centering
    \begin{tabular}{llllll}
    Dataset & $\sigma$ & Terminal Time $T$ & $\Delta t$ & \# Particles $N$ & \# Irregular Observation $N^\prime$\\
    \toprule
    Kuramoto           &  1  & 5   & 0.05    & \multirow{6}*{20}  & \multirow{6}*{20} \\ 
    Fitzhugh-Nagumo    & 0.3 & 5   & 0.05    &                    &  \\ 
    Opinion Dynamic    & 0.5 & 100 & 1.0     &                    &  \\ 
    Mean-field Atlas    & 1   & 5   & 0.05    &                    &  \\ 
    Ornstein-Uhlenbeck & 1   & 5   & 0.05    &                    &  \\ 
    Circles    & 1   & 5   & 0.05    &                    &  \\ 
    OU with Jumps            & 1   & 5   & 0.05    &   100              &  Not Applicable\\ 
    \bottomrule
    \end{tabular}
\vspace{5pt}\caption{Synthetic time series parameters}
\label{tab:synthetic_data}
\end{table}

\subsubsection{Real time series data}

\paragraph{EEG data.}
We used the 1-dimensional EEG data provided by \cite{zhang1995eeg}. Specifically, the EEGs recorded with stimulus 1. Each subject has 64 time series, and each time series has 256 timesteps. We used the following subject-run combinations for Non-Alcoholics EEG (NA-EEG): co2c0000362-076, co2c0000367-052, co2c0000338-016, co2c0000394-044, co2c0000348-016; and these subject-run combinations for Alcoholics EEG (A-EEG): co2a0000364-000, co2a0000372-014, co2a0000396-112, co2a0000411-064, co2a0000390-030.
We did not perform any further preprocessing on this dataset.

\paragraph{Crowd trajectory data.}
We used the 2-dimensional Crowd trajectory data provided by \cite{kothari2021human}. We used the crowd\textunderscore students001.ndjson dataset and extracted every trajectory that starts at the 10th frame with a trajectory length longer than 50 frames. We truncate the data by taking only up to the 50th frame. This leaves 23 pedestrian trajectories to train with. Trajectories that did not start at the 10th frame, or shorter than 50 frames are discarded. 

\paragraph{Chemotaxi data.}
We used the 3-dimensional Chemotaxi data provided by \cite{grognot2021multiscale}. We used \emph{V\textunderscore0208} for C.Crescentus and \emph{V\textunderscore{MeAsp1}\textunderscore{0511}} for E.Coli. The time series are truncated to the first 100 timesteps. Particles with less than 100 timesteps recorded are discarded. 

All time series data are split into 0.8 training, 0.1 validation, and 0.1 testing particles. 

\subsubsection{Generative data} \label{sec:gendata}
We are interested in estimating a flow between a Gaussian and a target distribution described by the nonlinear Fokker-Planck equation.
We thus sample a batch of $N=100$ particles from the initial condition $\mathcal{N}(0, I_{d \times d})$ where $d$ is the dimensions of the process. 
We then sample the same number of particles from different terminal conditions corresponding to the different datasets, i.e. Gaussian mixture and UCI datasets: Power, Miniboone, Hepmass, Gas and Cortex. 
To create the training dataset, we randomly match the particles from the initial condition to the particles from the terminal condition, then sample $N_{BB}=30$ Brownian bridges between each initial-terminal condition pair for $t\in[0,T], T=0.1, \Delta t = 0.002$. 

\paragraph{Eight Gaussians.} 
In the case of two dimensions, the terminal condition is an eight Gaussian mixture with means $\mu \in \left\{ [0,2],[0,-2],[2,0],[0,-2],[\sqrt{2},\sqrt{2}],[\sqrt{2},-\sqrt{2}],[-\sqrt{2},\sqrt{2}],[-\sqrt{2},-\sqrt{2}] \right\}$ and variance $I_{d\times d}$. 
For dimensions 10, 30, 50, 100, $\mu$ is repeated 5, 15, 25, and 50 times. 
\paragraph{Real data.} 
For \textsc{Power}, \textsc{Miniboone}, \textsc{Hepmass}, and \textsc{Gas}, we follow the preprocessing of~\citet{grathwohl2018scalable}. 
For the \textsc{Cortex} data, we normalize the data by 
subtracting the mean and dividing by the standard deviation. 

\subsection{Hyperparameter settings}
\label{sec:exp_params}
Since our goal is to determine the effect of different architectures, we try to control such that all architectures have similar number of parameters. 
The details of the hyperparameters are specified in Tables~\ref{tab:synthetic_nn_1},~\ref{tab:EEG_nn},~\ref{tab:Crowd_nn}, ~\ref{tab:Chemo_nn} and~\ref{tab:realgen_nn} for the different datasets.
The learned measure $W_0$ in the IM architecture was modeled as an additional fully connected layer. 
The marginal law $\hat{p}_t$ in the ML architecture was modeled with GLOW~\citep{Kingma2018glow} with an additional conditioning on time. 

For the MV-SDE models, we used the AdamW optimizer with a learning rate of $1 \times 10^{-4}$, $\epsilon=1 \times 10^{-4}$ and exponential decay $\gamma=0.9998$ for all experiments, except EEG where the learning rate was $1 \times 10^{-3}$. For the DeepAR models, the learning rate was $1 \times 10^{-3}$. 

The batch sizes used were $10, 5, 5, 10$ and $200$ for the synthetic time series, EEG, Crowd Trajectory, Chemotaxis and generative modeling experiments respectively. 

The models were trained for $500, 500, 200, 2000$, and $500$ epochs for the synthetic time series, EEG, Crowd Trajectory, Chemotaxis and generative modeling experiments.

\begin{table}[h!]
\small
    \centering
    \begin{tabular}{@{}lllll@{}}
        \toprule
    Architecture & Hidden Layers & Layer Size & Activation & \# of Parameters  \\ \midrule 
    MLP (It\^o)   & 8  & 128 & \multirow{4}*{$\mathrm{LeakyReLU}$} & 132740\\ 
    EM & $\varphi$: 4, $f$: 4 & 128, 128 &  & 133638 \\ 
    IM & $\varphi$: 4, $f$: 4, $W_0$: 1 & 128, 128, 128  &  & 134022 \\
    ML   & $\varphi$: 4, $f$: 4, $\hat{p}_t$: 1 & 128, 128, 32   & s: $\tanh$, t: $\mathrm{ReLU}$ & 136152   \\ \midrule
    MLP (It\^o)   & 4  & 128 & \multirow{4}*{$\mathrm{LeakyReLU}$} & 66820\\ 
    EM & $\varphi$: 2, $f$: 2 & 128, 128 &  & 67718 \\ 
    IM & $\varphi$: 2, $f$: 2, $W_0$: 1 & 128, 128, 128  &  & 67846 \\
    ML   & $\varphi$: 2, $f$: 2, $\hat{p}_t$: 1 & 128, 128, 32   & s: $\tanh$, t: $\mathrm{ReLU}$ & 70232  \\ \bottomrule  
    \end{tabular}
\vspace{5pt}\caption{Hyperparameter specification for synthetic time series data experiments and synthetic generative modeling experiments. The first set of hyperparameter settings are for: Kuramoto, Opinion Dynamic, Mean-field Atlas, Jump Diffusions, and Eight Gaussians. The second set of hyperparameter settings are for: Fitzhugh-Nagumo, It\^o-OU, It\^o-Circles. 
For the jump diffusion experiments, the $\mathrm{LeakyReLU}$ activation on the EM architecture led to diverging behavior, while $\mathrm{tanh}$ did not. We then changed the activation to $\mathrm{tanh}$ for more stable behaviour. }
\label{tab:synthetic_nn_1}
\end{table}

\begin{table}[h!]
\small
    \centering
    \begin{tabular}{lllll}
    Architecture & Modules: Hidden Layers & Layer Size & Activation & \# of Parameters \\
    \toprule
    MLP (It\^o)   & 10     & 64    & \multirow{4}*{$\mathrm{LeakyReLU}$} & 141858\\ 
    EM & $\varphi$: 4, $f$: 4 & 64, 64    &         & 133795  \\ 
    IM & $\varphi$: 4, $f$: 4, $W_0$: 1 & 64, 64, 512 &  & 134371        \\
    ML   & $\varphi$: 4, $f$: 4, $\hat{p}_t$: 3 & 64, 64, 32      &  s: $\tanh$, t: $\mathrm{ReLU}$   & 135960       \\ \midrule
    DeepAR-LSTM & \multirow{3}*{3} & 64 & \multirow{3}*{$\mathrm{LeakyReLU}$} & 117894\\
    DeepAR-RNN &  & 130 & & 120516\\
    DeepAR-GRU &  & 80 & & 137526\\
    DeepAR-TR & Enc: 8, Dec: 8 & 512 & ReLU & 298082 \\
    \bottomrule
    \end{tabular}
\vspace{5pt}\caption{Hyperparameter specification for real time series data - EEG experiments. For the DeepAR models, we used a window size of 20. }
\label{tab:EEG_nn}
\end{table}

\begin{table}[h!]
\small
    \centering
    \begin{tabular}{lllll}
    Architecture & Modules: Hidden Layers & Layer Size & Activation & \# of Parameters \\
    \toprule
    MLP (It\^o)   & 8     & 128    & \multirow{4}*{$\mathrm{LeakyReLU}$} & 133126\\ 
    EM & $\varphi$: 4, $f$: 4 & 128, 128    &         & 134409     \\ 
    IM & $\varphi$: 4, $f$: 4, $W_0$: 1 & 128, 128, 128 &  & 134921        \\
    ML  & $\varphi$: 4, $f$: 4, $\hat{p}_t$: 3 & 128, 128, 128      &  s: $\tanh$, t: $\mathrm{ReLU}$   &  239724     \\ \midrule
    DeepAR-LSTM & \multirow{3}*{3} & 64 & \multirow{3}*{$\mathrm{LeakyReLU}$} & 117123\\
    DeepAR-RNN &  & 130 & & 119733\\
    DeepAR-GRU &  & 80 & & 136723\\
    DeepAR-TR & Enc: 4, Dec:4 & 256 & ReLU & 81638 \\
    \bottomrule
    \end{tabular}
\vspace{5pt}\caption{Hyperparameter specification for real time series data - Chemotaxi experiments. For the DeepAR models, we used a window size of 20.}
\label{tab:Chemo_nn}
\end{table}

\begin{table}[h!]
\small
    \centering
    \begin{tabular}{lllll}
    Architecture & Modules: Hidden Layers & Layer Size & Activation & \# of Parameters \\
    \toprule
    MLP (It\^o)   & 2     & 64    & \multirow{4}*{$\mathrm{LeakyReLU}$} & 8644\\ 
    EM & $\varphi$: 1, $f$: 1 & 64, 64    &         & 9094     \\ 
    IM & $\varphi$: 1, $f$: 1, $W_0$: 1 & 64,64,64 &  & 9286        \\
    ML  & $\varphi$: 1, $f$: 1, $\hat{p}_t$: 1 & 64, 64, 128      &  s: $\tanh$, t: $\mathrm{ReLU}$   &  43672     \\ \midrule
    DeepAR-LSTM & \multirow{3}*{3} & 64 & \multirow{3}*{$\mathrm{LeakyReLU}$} & 117510\\
    DeepAR-RNN &  & 130 & & 120126\\
    DeepAR-GRU &  & 80 & & 137126\\
    DeepAR-TR & Enc: 8, Dec:8 & 256 & ReLU & 298148 \\
    \bottomrule
    \end{tabular}
\vspace{5pt}\caption{Hyperparameter specification for real time series data - Crowd Trajectory experiments. For the DeepAR models using a window size of 10.}
\label{tab:Crowd_nn}
\end{table}

\begin{table}[h!]
\footnotesize
    \centering
    \resizebox{\textwidth}{!}{%
    \begin{tabular}{lllll}
    Architecture & Modules: Hidden Layers & Layer Size & Activation & \# of Parameters  \\
    \toprule
    MLP (It\^o)   & 8  & 128 & \multirow{4}*{$\mathrm{tanh}$} & 133900 $\sim$ 151960\\ 
    EM & $\varphi$: 4, $f$: 4 & 64, 128 &  & 86098 $\sim$ 122148 \\ 
    IM & $\varphi$: 4, $f$: 4, $W_0$: 1 & 64, 128, 128 &  & 86930 $\sim$ 131940\\
    ML   & $\varphi$: 4, $f$: 4, $\hat{p}_t$: 1 & 64, 128, 32   & s: $\tanh$, t: $\mathrm{ReLU}$ & 89208 $\sim$ 146048  \\ \midrule
    MAF   & 4 & 128 & $\mathrm{ReLU}$ & 75872 $\sim$ 184512\\
    WGAN & Gen: 4, Dis: 3  & Gen: [64, 128, 256], Dis: 256 & $\mathrm{LeakyReLU}$ & 114759 $\sim$ 150669\\
    VAE & Enc: 4, Dec: 4 & 128, 256, latent dim: 50 & $\mathrm{LeakyReLU}$ & 88682 $\sim$ 124592\\
    Score-Based & 8 & 128 & $\mathrm{SiLU}$ & 117318 $\sim$ 135308\\
    \bottomrule
    \end{tabular}
    }
\vspace{5pt}\caption{Hyperparameter specification for generative modeling experiments: Power, Miniboone, Hepmass, Gas and Cortex. The number of parameters depends on the dimension of the data. The hyperparameter specification for the generative modeling experiments with Eight Gaussians follow that of Table~\ref{tab:synthetic_nn_1}. }
\label{tab:realgen_nn}
\end{table}

\clearpage

\subsection{Additional figures and tables}
We provide a series of additional figures to qualitatively illustrate the differences between the proposed architectures and baselines.

\subsubsection{Ablation on IM architecture width}
\label{sec:w0_ablation}
We conduct a series of ablations on the width of the IM architecture.
These ablations are performed on the synthetic datasets of the Kuramoto model and the Fitzhugh-Nagumo model. 
The results are presented in Figure~\ref{fig:kura_ablat} and~\ref{fig:fitz_ablat}.

\begin{figure}[h!]
    \centering
    \includegraphics[width=\textwidth]{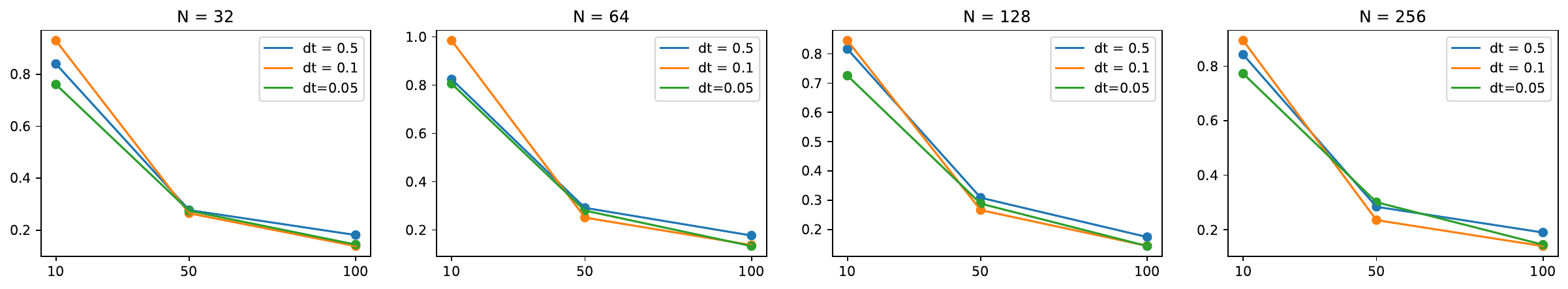}
    \caption{Ablation on different IM architecture widths $N=32, 64, 128, 256$ for the Kuramoto model, with different time grid size $d t$ and different number of training particles.}
    \label{fig:kura_ablat}\vspace{5pt}
    \includegraphics[width=\textwidth]{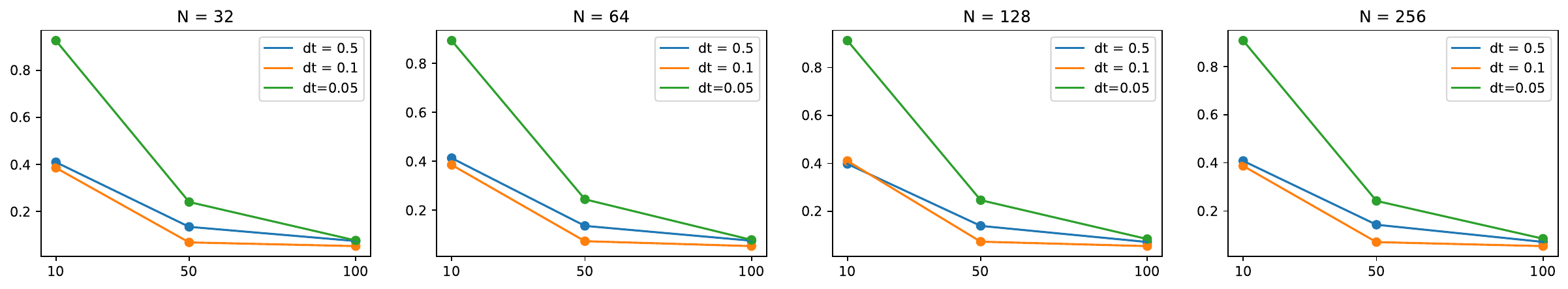}
    \caption{Ablation on different IM architecture widths $N=32, 64, 128, 256$ for the Fitzhugh-Nagumo model, with different time grid size $d t$ and different number of training particles.}
    \label{fig:fitz_ablat}
\end{figure}

\subsubsection{Ablation on IM architecture number of training particles}
\label{sec:numparticle_ablation}
In Table \ref{tab:ks_numpablation} we show that as the number of training particles grows, the KS distance will approach 0. 
Note that since KS-statistic only applies to 1-dimensional ECDFs.
We consider a 1-d Kuramoto simulation with the same parameters and no additional observation noise. 

\begin{table}[h!]
\centering
\begin{tabular}{@{}lllll@{}}
    \toprule 
 $N$        & mean  & 75\%          & 90\% & KS\\\midrule
10      & 0.0481 (0.023) & 0.0794 (0.041) & 0.1021 (0.042) & 0.1222 (0.041)        \\
50      & 0.0272 (0.018) & 0.0434 (0.033) & 0.0609 (0.038) & 0.0792 (0.038)        \\
100      & 0.0177 (0.006) & 0.0278 (0.009) & 0.0383 (0.009) & 0.0544 (0.008)        \\
500      & 0.0162 (0.005) & 0.0249 (0.01) & 0.0352 (0.01) & 0.0519 (0.009)        \\\bottomrule
\end{tabular}
\caption{ECDF distance of 1d Kuramoto under $N = 10, 50, 100$ and $500$ training particles and $500$ test particles. The values represent the distance between the ECDFs averaged over all time marginals at the mean, 75th, 90th, and 100 percentiles.  }
\label{tab:ks_numpablation}
\end{table}

\subsubsection{Synthetic data experiments}
We provide tables with variance for the synthetic dataset result showcased in the main text. We also provide a KS-distance table for the 1-dimensiaonl mean-field atlas experiment. KS-distance is only applicable in 1-dimensional cases. 
\begin{table}[h!]
\centering
\caption{Synthetic dataset results with noise level standard deviation 0.1}\vspace{5pt}
    \footnotesize
    \label{tab:syntheticdataTS_0.1}
    \begin{tabular}{lllllll}
    \toprule 
                        & Kuramoto       & Fitzhugh  & OD       & MA  & OU & Circle  \\ \midrule
    MLP (It\^o) & 0.56 (0.081) & 0.699 (0.426) &0.048 (0.013)& 2.14 (0.142) & 0.098 (0.043) & 1.351 (1.979)\\ 
    IM          & 0.448 (0.075)& 0.601 (0.422) &0.039 (0.011)& 1.208 (0.147)& 0.128 (0.047) & 1.592 (2.38) \\
    ML          & 0.428 (0.095)& 0.639 (0.395) &0.042 (0.012)& 1.519 (0.236)& 0.101 (0.038) & 1.481 (2.237) \\
    EM          & 0.383 (0.085)& 0.606 (0.389) &0.036 (0.01) & 1.359 (0.2)  & 0.097 (0.038) & 1.562 (2.476) \\\bottomrule
    \end{tabular}
\end{table}

\begin{table}[h!]
\centering
\caption{Synthetic dataset results with noise level standard deviation 0.5}\vspace{5pt}
    \footnotesize
    \label{tab:syntheticdataTS_0.5}
    \begin{tabular}{lllllll}
    \toprule 
                        & Kuramoto       & Fitzhugh  & OD       & MA  & OU & Circle  \\ \midrule
    MLP (It\^o) & 0.578 (0.124)& 0.734 (0.489) &0.047 (0.012)& 2.133 (0.156)& 0.1 (0.042)   & 1.334 (1.948)\\ 
    IM          & 0.45 (0.074) & 0.617 (0.415) &0.039 (0.012)& 1.223 (0.125)& 0.128 (0.046) & 1.592 (2.368)\\
    ML          & 0.397 (0.075)& 0.605 (0.408) &0.042 (0.011)& 1.518 (0.248)& 0.1 (0.035)   & 1.564 (2.543) \\
    EM          & 0.373 (0.075)& 0.612 (0.383) &0.038 (0.01) & 1.347 (0.165)& 0.106 (0.036) & 1.535 (2.535) \\\bottomrule
    \end{tabular}
\end{table}

\begin{table}[h!]
\centering
\caption{Synthetic dataset results with noise level standard deviation 1.0}\vspace{5pt}
    \footnotesize
    \label{tab:syntheticdataTS_1.0}
    \begin{tabular}{lllllll}
    \toprule 
                        & Kuramoto       & Fitzhugh  & OD       & MA  & OU & Circle  \\ \midrule
    MLP (It\^o) & 0.653 (0.067)& 0.897 (0.503)&0.059 (0.009)& 2.159 (0.2)  & 0.481 (0.065) &2.303 (2.039)\\ 
    IM          & 0.646 (0.065)& 0.878 (0.522)&0.055 (0.012)& 1.65 (0.232) & 0.559 (0.062) &2.658 (2.142)\\
    ML          & 0.601 (0.112)& 0.882 (0.527)&0.049 (0.009)& 1.748 (0.224)& 0.529 (0.055) &2.308 (2.252) \\
    EM          & 0.592 (0.077)& 0.893 (0.523)&0.04 (0.009) & 1.652 (0.272)& 0.536 (0.055) &2.394 (2.25) \\\bottomrule
    \end{tabular}
\end{table}

\begin{table}[h!]
\centering
\caption{ECDF distance of mean-field atlas process under 20 training particles and 20 irregular time samples with different noises. The result is tested with 100 test particles and averaged across 10 runs.}
\label{tab:ksdist_mfatlas}
\begin{tabular}{@{}lllll@{}}
    \toprule 
         & mean ECDF dist  & 75\% ECDF dist & 90\% ECDF dist & KS\\\midrule
Noise = 1.0 & & &  \\ \midrule
 MLP & 0.1 (0.01) & 0.16 (0.03) & 0.27 (0.03) & 0.35 (0.03)\\
 IM  & 0.05 (0.01) & 0.07 (0.02) & 0.13 (0.03) & 0.21 (0.03)\\
 ML  & 0.05 (0.02) & 0.07 (0.03) & 0.14 (0.05) & 0.24 (0.03)\\
 EM  & 0.05 (0.01) & 0.07 (0.03) & 0.15 (0.04) & 0.24 (0.03) \\\midrule
Noise = 0.5 & & & \\ \midrule
 MLP & 0.09 (0.01) & 0.16 (0.02) & 0.26 (0.02) & 0.34 (0.03)\\
 IM  & 0.03 (0.01) & 0.04 (0.01) & 0.08 (0.02) & 0.15 (0.02)\\
 ML  & 0.04 (0.01) & 0.06 (0.02) & 0.12 (0.03) & 0.21 (0.03)\\
 EM  & 0.03 (0.01) & 0.04 (0.01) & 0.07 (0.02) & 0.14 (0.02)\\\midrule
Noise = 0.1 & & & \\ \midrule
 MLP  & 0.1 (0.01) & 0.16 (0.03) & 0.26 (0.03) & 0.34 (0.02)\\
 IM   & 0.03 (0.01) & 0.04 (0.01) & 0.08 (0.02) & 0.14 (0.02)\\
 ML   & 0.04 (0.01) & 0.06 (0.02) & 0.13 (0.03) & 0.21 (0.02)\\
 EM   & 0.03 (0.01) & 0.04 (0.01) & 0.08 (0.02) & 0.14 (0.02)\\\bottomrule
\end{tabular}
\end{table}

For a better sense of the different synthetic datasets and each model's ability in recovering the drift, we provide a figure that qualitatively compares the architectures' performances in Figure~\ref{fig:synth_noise_result}.
We additionally show the learnt gradient flow for the Kuramoto model in Figure~\ref{fig:kura_gradient}.
\begin{figure}[h!]
    \centering
    \includegraphics[width=\textwidth]{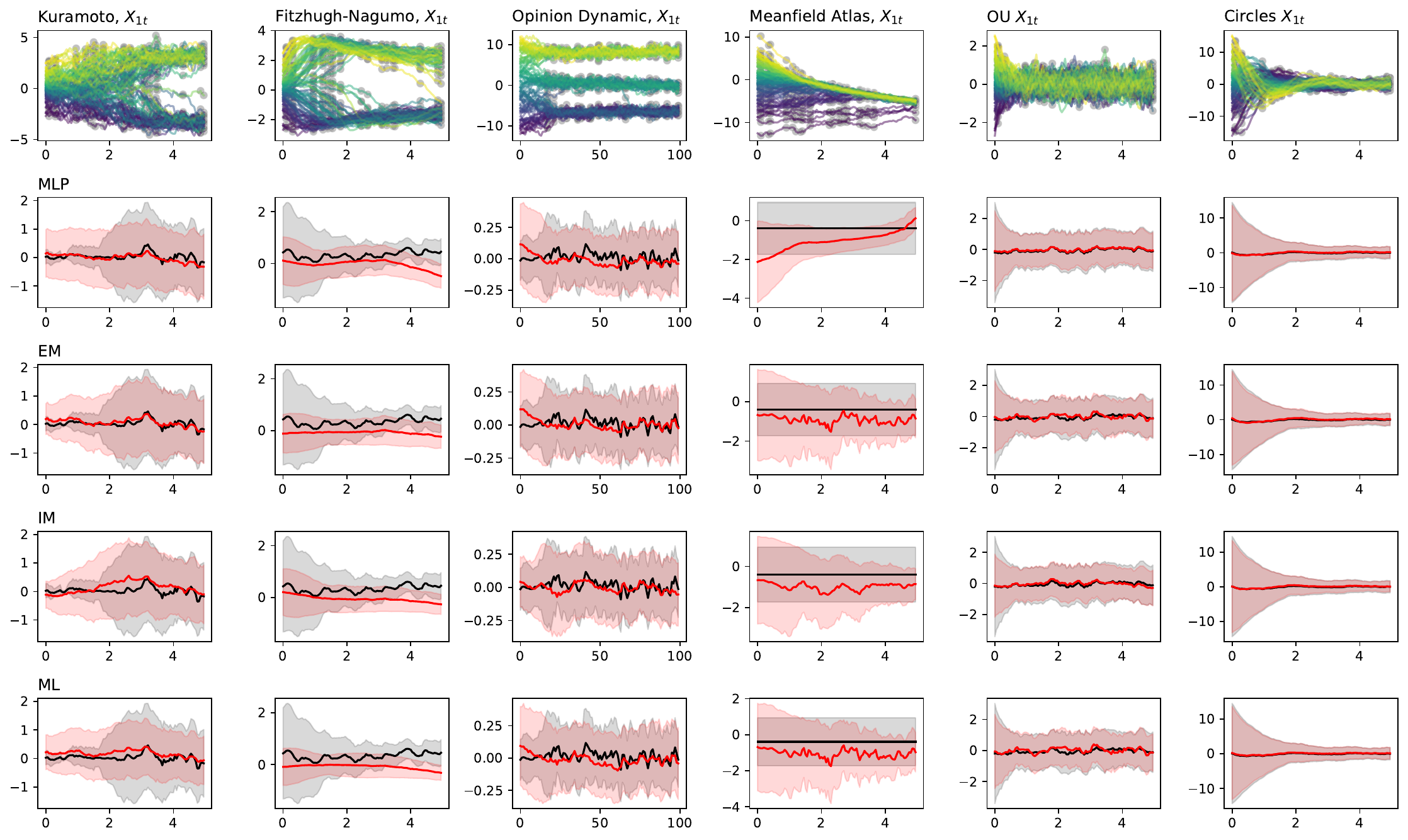}
    \caption{Synthetic data experiments and estimated drifts, only the first dimension is shown. First row: sampled trajectories, grey scattered circles indicate irregular time observations. Rows 2-5: estimated drifts by the MLP (It\^o), EM, IM, ML architectures. Black is truth, red is estimated. The models are trained with additional Gaussian observation noise of SD = 0.1.}
    \label{fig:synth_noise_result}
\end{figure}
\begin{figure*}[h!]
    \centering
     \begin{subfigure}[t]{0.19\textwidth}
         \centering
         \includegraphics[width=\textwidth]{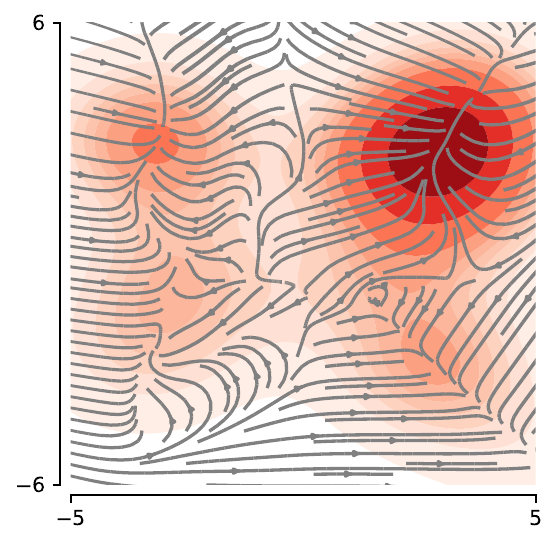}
         \caption{MLP}
         \label{fig:MLP_kura_sup}
     \end{subfigure}
     \hfill
     \begin{subfigure}[t]{0.19\textwidth}
         \centering
         \includegraphics[width=\textwidth]{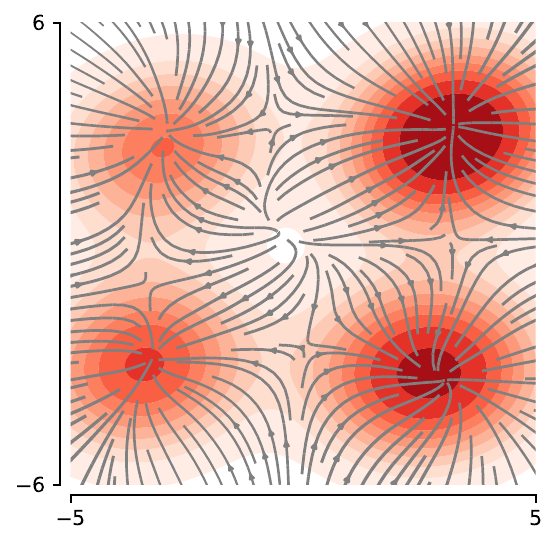}
         \caption{EM}
         \label{fig:Xt_kura_sup}
     \end{subfigure}
     \hfill
     \begin{subfigure}[t]{0.19\textwidth}
         \centering
         \includegraphics[width=\textwidth]{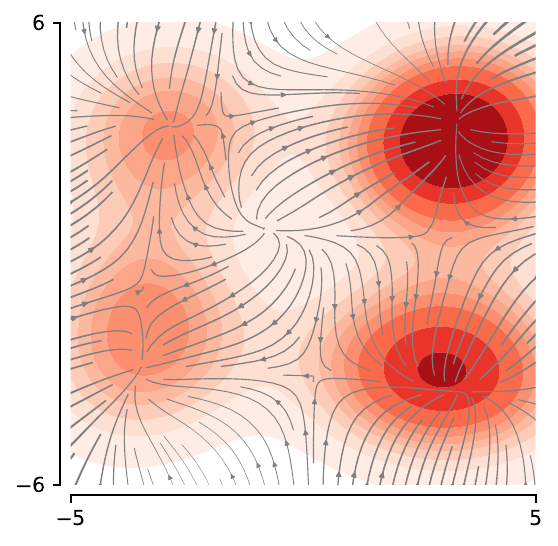}
         \caption{IM}
         \label{fig:W0_kura_sup}
     \end{subfigure}
     \hfill
     \begin{subfigure}[t]{0.19\textwidth}
         \centering
         \includegraphics[width=\textwidth]{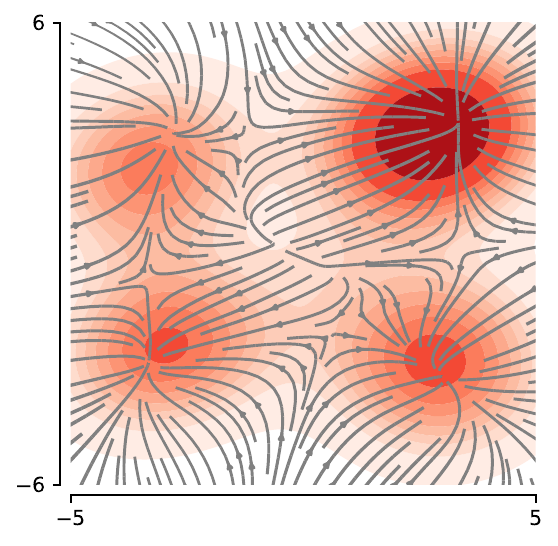}
         \caption{ML}
         \label{fig:gen_kura_sup}
     \end{subfigure}
     \hfill
     \begin{subfigure}[t]{0.19\textwidth}
         \centering
         \includegraphics[width=\textwidth]{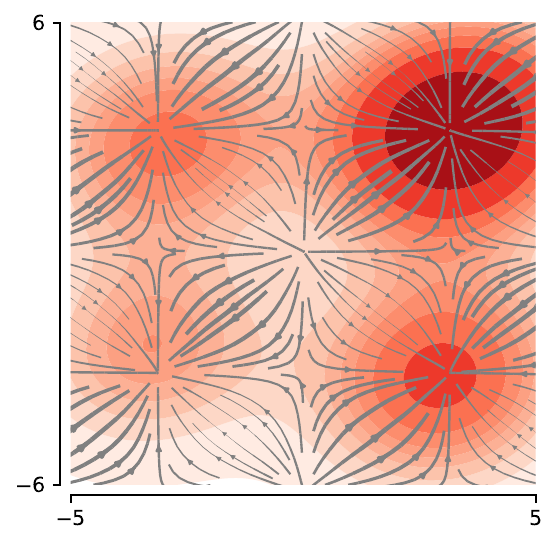}
         \caption{Truth}
         \label{fig:true_kura_sup}
     \end{subfigure}
        \caption{Estimated gradient flow of Kuramoto Model at terminal time. The colors correspond to the density of generated samples at terminal time. The models are trained with additional Gaussian observation noise of SD = 0.1.} 
        \label{fig:kura_gradient}
\end{figure*}

\clearpage

\subsubsection{Real data experiments}
\label{sec:realTSForecast}
We provide an additional experiment on brain EEG data recorded for both alcoholic and non-alcoholic subject. 
We also extend the time series experiments in the main paper to different types of generation and forecasting. 
For generation, we begin with an initial condition at $T_0$ and generate a trajectory up to $T_{\text{forecast}}$. 
There are two ways to perform forecasting: i) given the initial condition at $T_0$, generate trajectories up to $T_{\text{forecast}}$; ii) given the training terminal condition at $T$, generate trajectories for $t \in [T, T_{\text{forecast}} ]$. In both cases, the dataset time steps are partitioned into 0.8 training, 0.2 forecasting.
We present the numerical results of generation and both types of forecasting in Tables ~\ref{tab:realdataTS_gen}, ~\ref{tab:realdataTS_fore1} and~\ref{tab:realdataTS_fore2}. We note that our methods perform on par with various DeepAR methods under both types of forecasting. We also present qualitative results with the learnt drifts in Figures~\ref{fig:EEG_mean_samples}, ~\ref{fig:Chemo_mean_samples}, and ~\ref{fig:crowd_traj_samples} for the EEG, Chemotaxis data, and Crowd Trajectory.
\begin{table}[h!]
\centering
\caption{Time series generation on held out trajectories. NA/A stands for non-alcoholics/alcoholics. Values in \textbf{bold} and \textit{italic} are best and second best.}
\label{tab:realdataTS_gen}
\begin{tabular}{@{}llll@{}}
\toprule 
& \multicolumn{2}{c}{CRPS $\downarrow$}   & \multicolumn{1}{c}{MSE $\downarrow$} \\ 
                & NA-EEG                  & A-EEG                  & Crowd Traj                 \\ \midrule
MLP (It\^o)     & 5.795 (1.53)            & 4.404 1.593            & 0.939 0.477 \\ 
IM              & 5.179 (1.24)            & \textit{4.163} 1.219   & \textit{0.811} 0.888 \\
ML              & \textbf{5.103} (1.092)  & \textbf{4.052} 1.006   & \textbf{0.51} 0.569 \\
EM              & \textit{5.174} (1.397)  & 4.252 1.348            & 0.877 0.442\\ \midrule
LSTM            & 6.129 (2.237)           & 5.679 2.557            & 1.803 0.373 \\
RNN             & 6.052 (2.286)           & 4.643 1.379            & 1.367 1.066 \\
GRU             & 6.201 (2.219)           & 6.176 2.726            & 1.133 0.318 \\
TR              & 5.986 (1.614)           & 4.295 1.357            & 3.278 1.078          \\\bottomrule
\end{tabular}
\end{table}
\begin{table}[h!]
\centering
\caption{Time series forecasting Type I. NA/A stands for Non-alcoholics/ Alcoholics.\\ \textbf{Bolded} values indicate best performance.}\vspace{5pt}
    \scriptsize
    \label{tab:realdataTS_fore1}
    \begin{tabular}{llllll}
    \toprule 
    & \multicolumn{2}{c}{CRPS $\downarrow$} & \multicolumn{3}{c}{MSE $\downarrow$} \\ \cmidrule(lr){2-6}
                    & NA-EEG                & A-EEG                 & Crowd Traj  & C.Cres                & E.Coli            \\ \midrule
    MLP (It\^o)     &30.087 (32.29)         & 7.837 (3.018)         & 3.502 2.017              & \textbf{0.296} (0.007) & \textbf{0.225} (0.007) \\ 
    IM           &8.346  (4.646)         & 5.438 (1.814)            & 3.458 3.609              & 0.307 (0.010)          & 0.23 (0.006) \\
    ML     &7.967  (4.542)         & 5.652 (1.515)                  & 4.006 6.418              & 0.312 (0.015)          & 0.245 (0.006) \\
    EM           &8.963  (4.309)         & 5.82 (1.818)             & 3.214 1.215              & 0.312 (0.019)          & 0.26 (0.013) \\ \midrule
    LSTM            &7.231  (3.051)         & 6.66 (3.948)          & 2.876 0.804              & 1.526 (0.324)          & 0.786 (0.386) \\
    RNN             &\textbf{6.993} (2.369) & 5.292 (2.317)         & 2.691 2.604              & 1.689 (1.107)          & 0.859 (0.115) \\
    GRU             &7.234  (2.75)          & 7.407 (4.494)         & 2.292 0.515              & 1.115 (0.406)          & 0.813 (0.337) \\
    TR              &7.354  (1.998)         & \textbf{5.122} (2.457)& 5.516 2.207              & 1.489 (0.362)          & 1.489 (0.362) \\\bottomrule
    \end{tabular}
\end{table}

\begin{table}[h!]
\centering
\caption{Time series forecasting Type II. NA/A stands for Non-alcoholics/ Alcoholics.\\ \textbf{Bolded} values indicate best performance.}\vspace{5pt}
    \scriptsize
    \label{tab:realdataTS_fore2}
    \begin{tabular}{lllll}
    \toprule 
    & \multicolumn{2}{c}{CRPS $\downarrow$} & \multicolumn{2}{c}{MSE $\downarrow$} \\ \cmidrule(lr){2-5}
                    & NA-EEG                & A-EEG                  & C.Cres                & E.Coli            \\ \midrule
    MLP (It\^o)     & 31.47 (35.659)          & 6.95  (2.640)        & \textbf{0.013} (0.0003) & \textbf{0.015 (0.0003)} \\ 
    IM           & 8.675 (5.638)           & \textbf{4.884} (1.687)  & 0.014 (0.0007)          & 0.016 (0.0003) \\
    ML     & 8.747 (5.677)           & 4.907 (1.490)                 & 0.015 (0.0007)          & 0.015 (0.0005) \\
    EM           & 8.938 (4.975)           & 5.403 (2.205)           & 0.015 (0.0013)          & 0.015 (0.0005) \\ \midrule
    LSTM            & 8.288 (3.142)           & 6.317 (4.207)          & 0.291 (0.0704)          & 0.163 (0.0353) \\
    RNN             & \textbf{7.002} (2.591)  & 5.296 (2.262)          & 1.455 (0.9367)          & 0.534 (0.191) \\
    GRU             & 7.019 (2.686)           & 6.044 (3.457)          & 0.397 (0.2134)          & 0.17  (0.054) \\
    TR              & 7.087 (2.208)           & 4.971 (2.643)          & 1.65 (0.1666)           & 1.65  (0.1666) \\\bottomrule
    \end{tabular}
\end{table}

\begin{figure}[h!]
    \centering
    \includegraphics[width=0.9\textwidth]{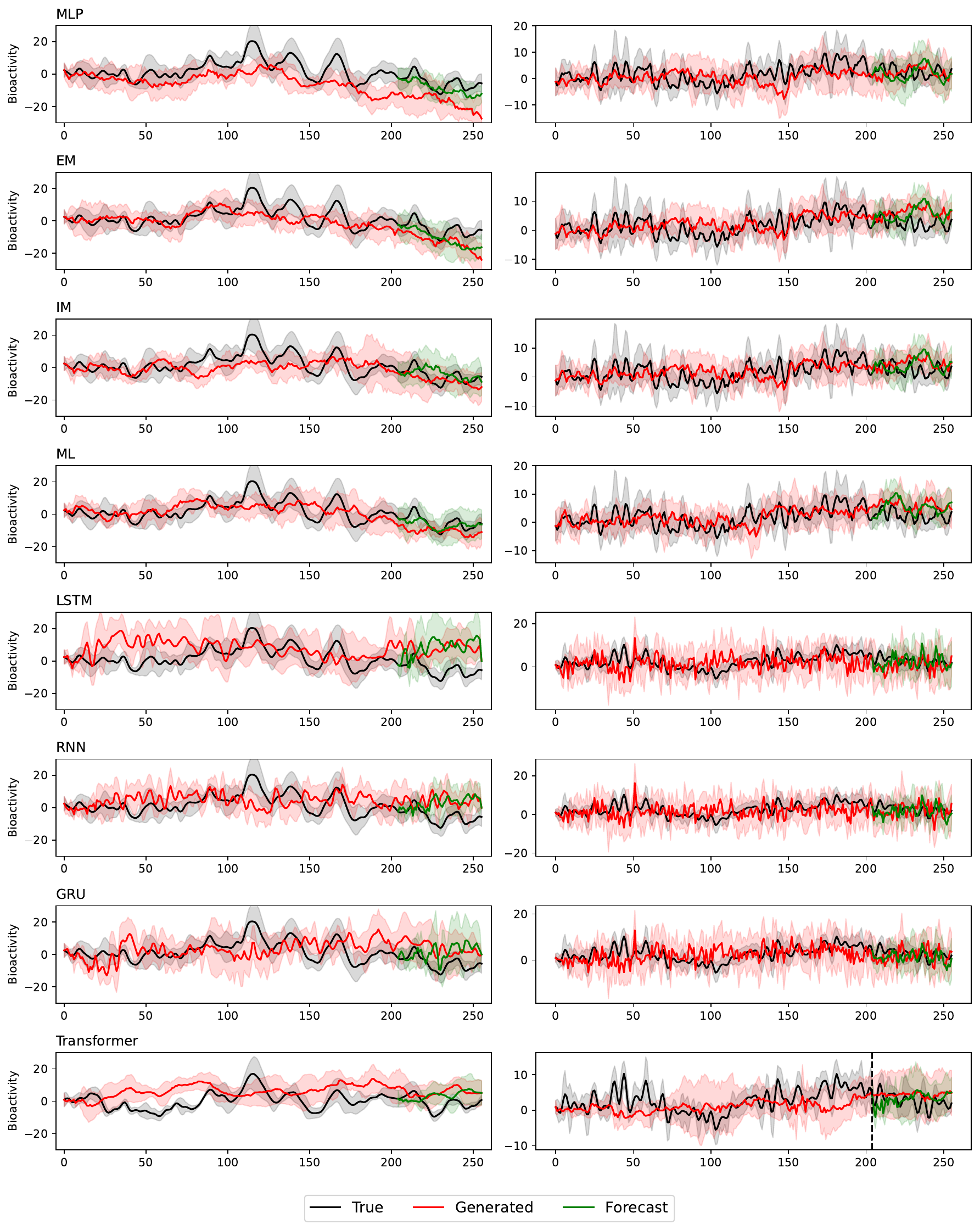}
    \caption{True, generated and forecasted trajectories on EEG dataset. Left: Non-alcoholics; Right: Alcoholics. The dashed vertical line at $t=205$ indicates the start of the forecast.
    The shaded region indicates $\pm$ one standard deviation of samples at each time step.}
    \label{fig:EEG_mean_samples}
\end{figure}
\begin{figure}[h!]
    \centering
    \includegraphics[width=0.9\textwidth]{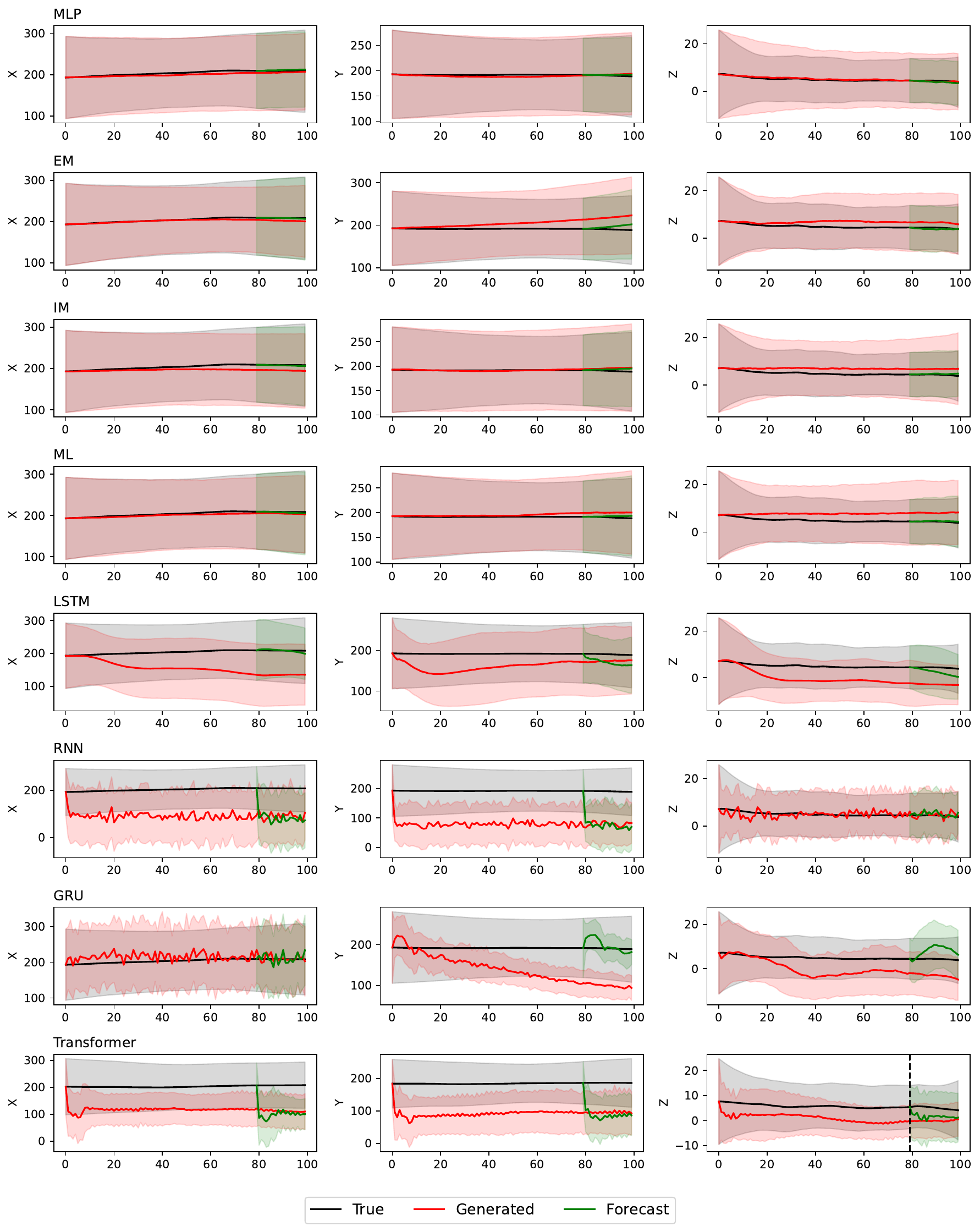}
    \caption{True, generated, and forecasted trajectories for C. Crescentus chemotaxis dataset. The dashed vertical line at $t=80$ indicates the start of the forecast. The shaded region indicates $\pm$ one standard deviation of samples at each time step. From left to right, the columns are movements in $x, y$ and $z$ directions. }
    \label{fig:Chemo_mean_samples}
\end{figure}

\begin{figure}[h!]
    \centering
    \includegraphics[width=0.9\textwidth]{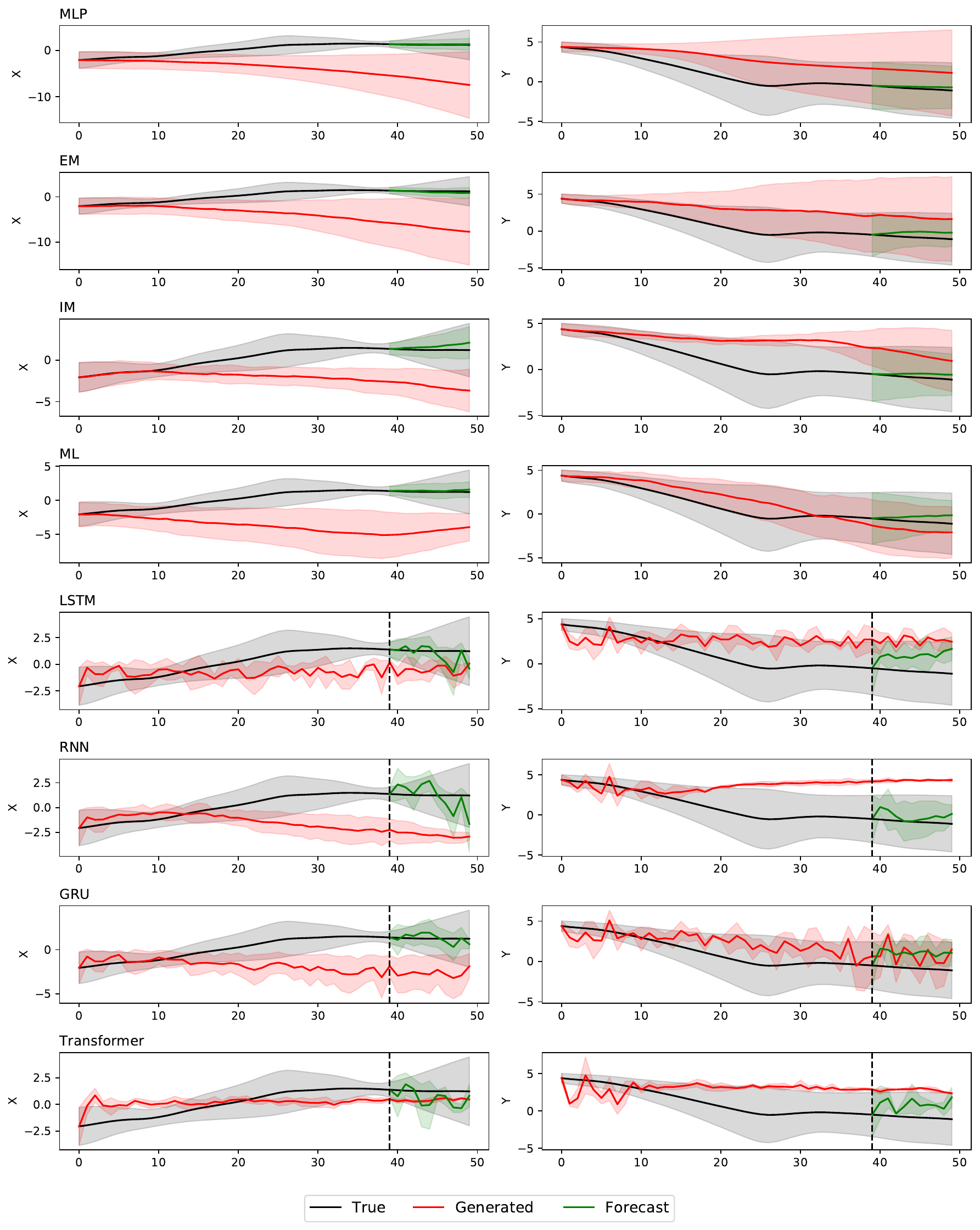}
    \caption{True, generated and forecast trajectories on Crowd Trajectory. The dashed vertical line at $t=39$ indicates the start of the forecast. The shaded region indicates $\pm$ one standard deviation of samples at each time step. From left to right, the columns are movements in $x$ and $y$ directions. }
    \label{fig:crowd_traj_samples}
\end{figure}

\clearpage

\subsubsection{Generative modeling experiments}
\label{sec:more_figs_gen}

Figure~\ref{fig:eightgauss_2d} shows 5 randomly selected 2-d projections of the 100-d mixture of Gaussians. 

\begin{figure}[h!]
    \centering
    \includegraphics[width=\textwidth]{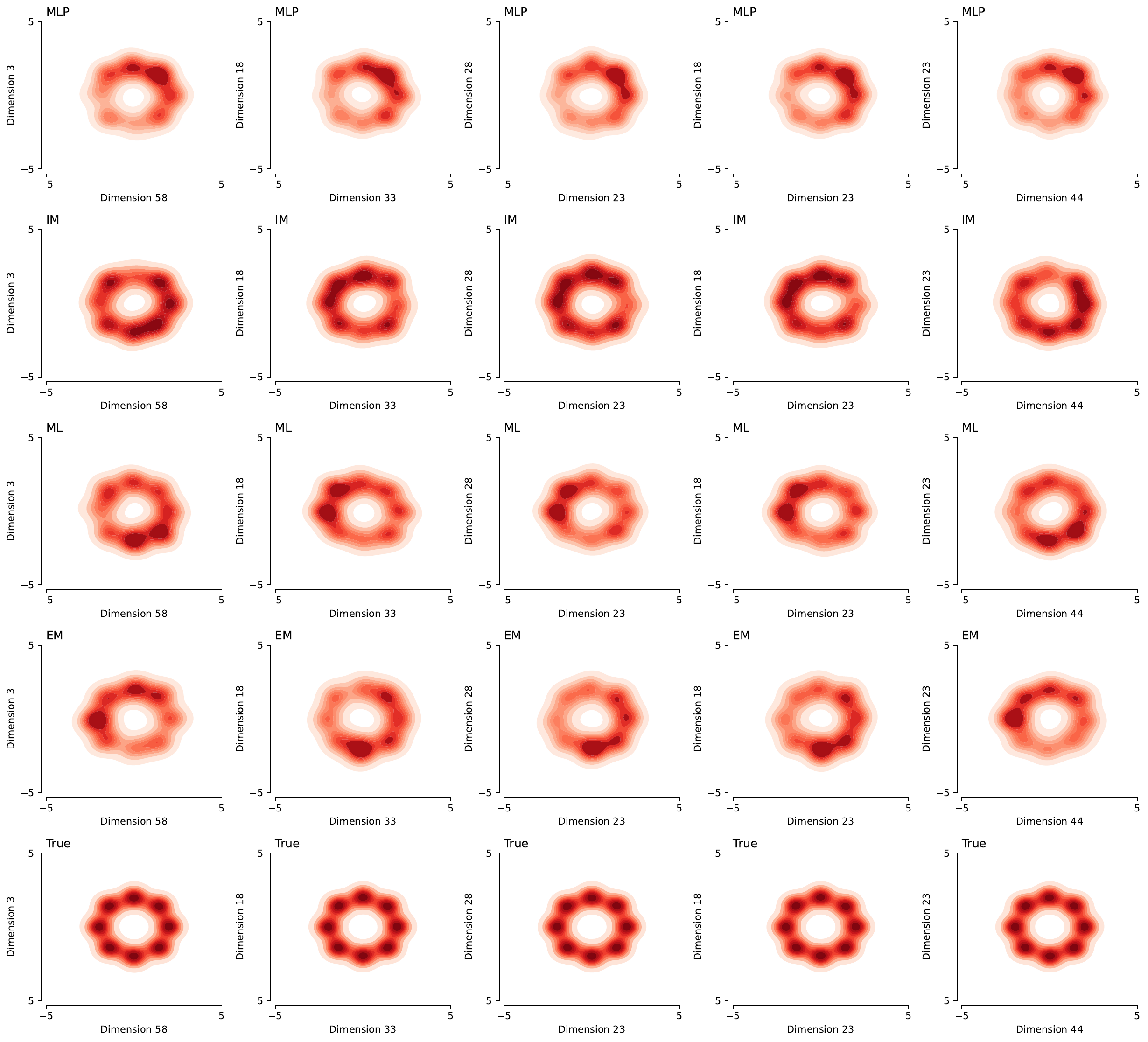}
    \caption{Five randomly selected 2-d projections of 100-d mixture of Gaussians. }
    \label{fig:eightgauss_2d}
\end{figure}

In addition to the eight Gaussian mixture and real data presented in the main paper, we present a few toy generative modeling experiments to better understand the different architectures.
\begin{figure}[h]
    \centering
     \begin{subfigure}[t]{.8\textwidth}
         \centering
         \includegraphics[width=\textwidth]{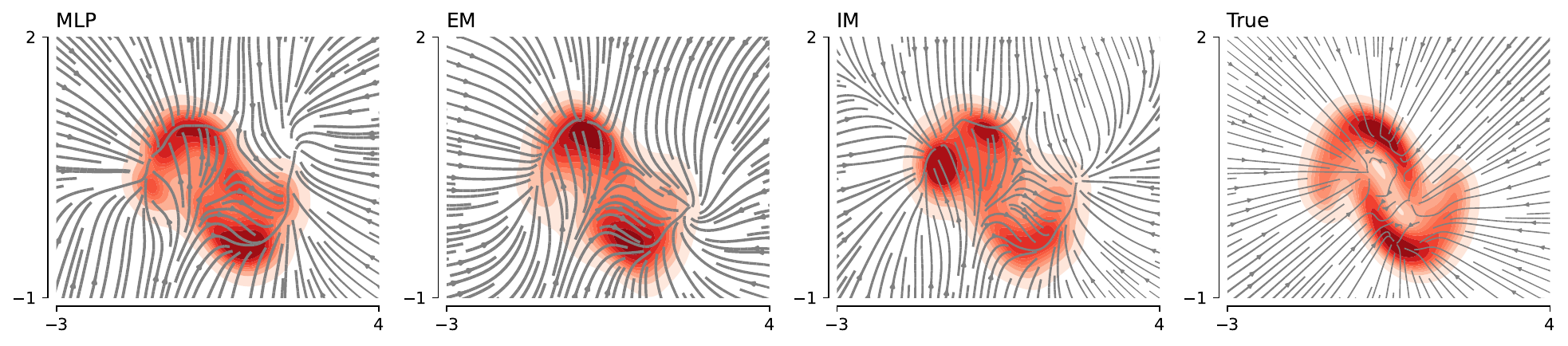}
         \caption{Two Moons}
         \label{fig:moons}
     \end{subfigure}
     \hfill
     \begin{subfigure}[t]{.8\textwidth}
         \centering
         \includegraphics[width=\textwidth]{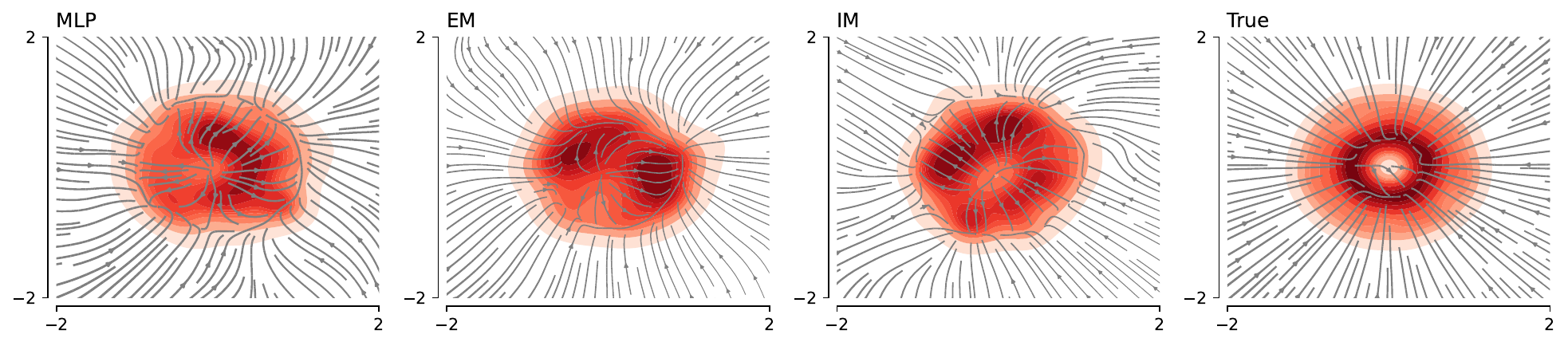}
         \caption{Two Circles}
         \label{fig:circles}
     \end{subfigure}
     \hfill
     \begin{subfigure}[t]{.8\textwidth}
         \centering
         \includegraphics[width=\textwidth]{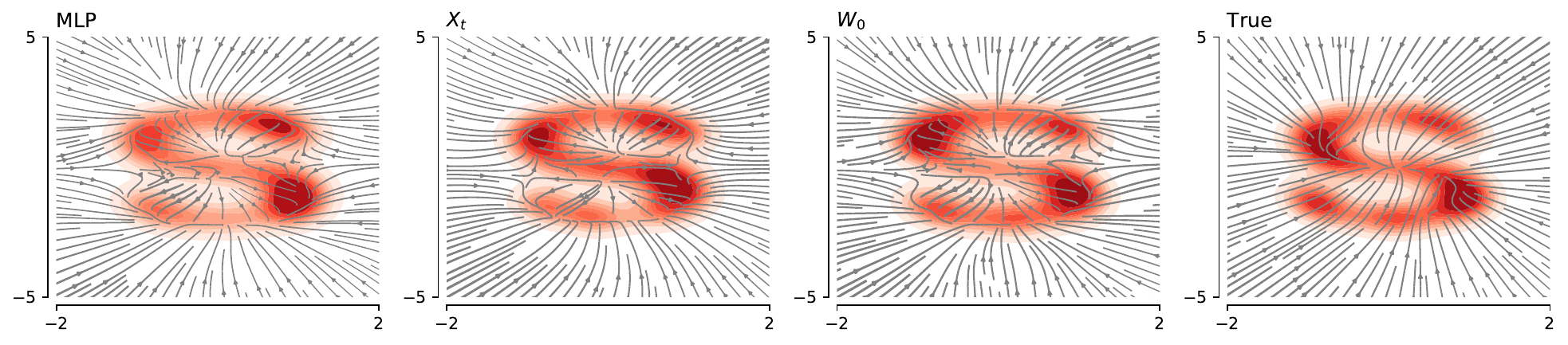}
         \caption{S Curve, 2-$d$}
         \label{fig:s}
     \end{subfigure}
     \hfill
     \begin{subfigure}[t]{.8\textwidth}
         \centering
         \includegraphics[width=\textwidth]{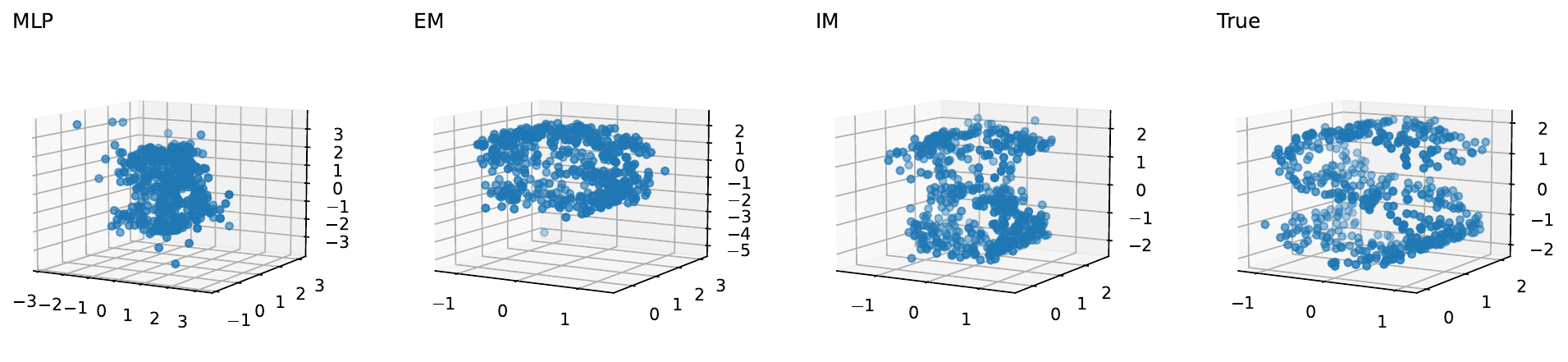}
         \caption{S Curve, 3-$d$}
         \label{fig:3d-s}
     \end{subfigure}
     \begin{subfigure}[t]{.8\textwidth}
         \centering
         \includegraphics[width=\textwidth]{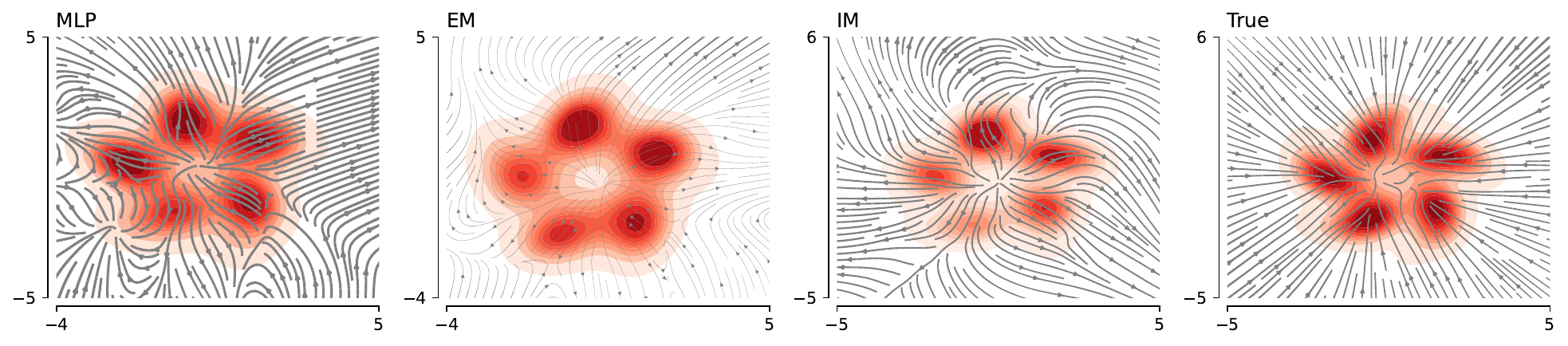}
         \caption{Pinwheel}
         \label{fig:swissroll}
     \end{subfigure}
     \hfill
     \begin{subfigure}[t]{.8\textwidth}
         \centering
         \includegraphics[width=\textwidth]{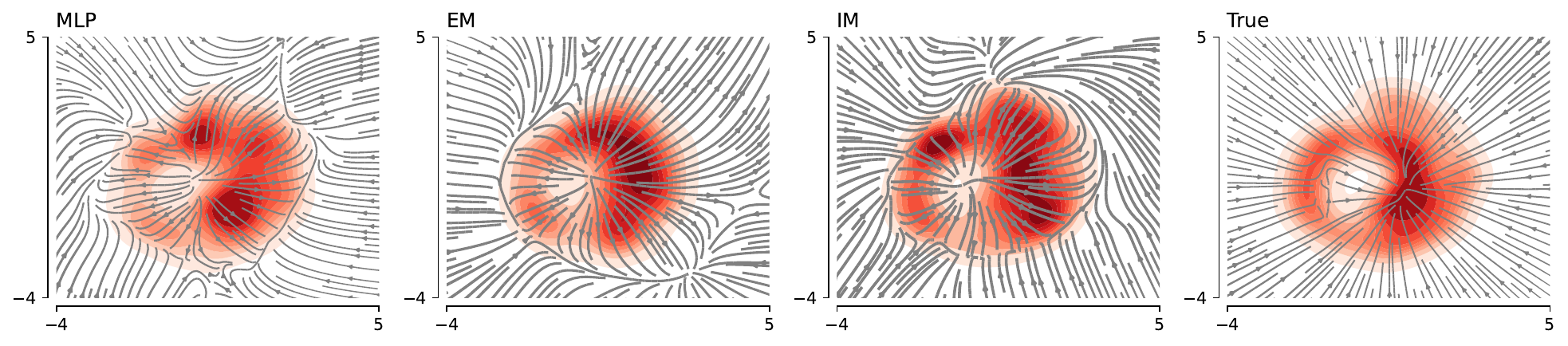}
         \caption{Swissroll}
         \label{fig:pinwheel}
     \end{subfigure}
        \caption{Estimated gradient flow at terminal time. The colors correspond to the density of generated samples at terminal time. 
        From left to right: MLP, EM, IM, true. From top to bottom: two moons, two circles, S-curve 2-d, S-curve 3-d, pinwheel, swissroll. The architectures were trained with the Brownian bridge estimator.} 
    \label{fig:generative_result_BB}
\end{figure}

\subsubsection{Additional experiments: linear Fokker–Planck}
\label{sec:fk}

For this set of experiments, the set up is similar to the generative modeling experiments detailed in Section~\ref{sec:gendata} where we map between a Gaussian distribution and target distributions of two moons, two circles, s-curves, and 3-dimensional s-curves. However, we consider a linearization of the PDE that governs the density and derive a likelihood for the target distribution based on the linearized PDE. We are mainly interested in the performance differences due to differences in architectures between MLP (It\^o), IM, and ML. A similar framework was considered in~\citet{huang2021variational} with respect to score-based generative models. We first derive the estimation procedure then show the results.

It is known that the flow satisfies the Fokker-Planck equation given by
\begin{equation}
\partial_t p_t = - \mathrm{div}(b(x,p_t,t) p_t(x)) + \frac{\sigma^2}{2} \nabla^2 p_t(x).
\label{eq:fk2}
\end{equation}
Falling back on the It\^o-SDE, where $b$ does not depend on $p_t$, the PDE is linear. We can then consider using the Feynman-Kac formula where the solution to~\eqref{eq:fk2} with $b$ independent of $p_t$ can be computed according to an expectation over sample paths $X_t$ that satisfy $\mathrm{d}X_t = b(\cdot) \mathrm{d} t + \sigma \mathrm{d}W_t$ such that
$$
p_T(x) = \mathbb{E}\left[ \exp \left ( \int_0^T - \mathrm{div} b( \cdot ) \, \mathrm{d} t\right) p_0(X_T) \; \bigg | \; X_0 = x \right ].
$$
We use Girsanov's theorem to transform the expectation over sample paths with drift to an  expectation under Brownian motion, i.e. over sample paths $X_t$ that satisfy $\mathrm{d}X_t = \sigma \mathrm{d}W_t$ with no drift and
$$
p_T(x) = \mathbb{E}\left[ \exp \left ( \int_0^T - \mathrm{div} b( \cdot ) \, \mathrm{d} t\right) p_0(X_T) \exp \left( \int_0^T b( \cdot ) \,\mathrm{d} X_t - \frac12 \int_0^T b^2(\cdot) \, \mathrm{d}t \right) \; \bigg | \; X_0 = x \right ].
$$
leading to an efficient Monte Carlo method for computing the probability. To maximize this likelihood, we can use Jensen's inequality to derive an ELBO which we optimize as
$$
\log p_T(x) \geq \mathbb{E}\left[   \int_0^T - \mathrm{div} b( \cdot ) \, \mathrm{d} t   + \log(p_0(X_T)) + \int_0^T b( \cdot ) \, \mathrm{d} X_t - \frac12 \int_0^T b^2(\cdot) \, \mathrm{d}t \; \bigg | \; X_0 = x \right ].
$$
The integrals are approximated using the forward Euler method and the parameters of $b$ are optimized for the set of observations. The results are given in Table~\ref{tab:linearfkGen}.
The results suggest that the proposed architectures do not decrease performance in the linear setting and sometimes provide slight improvements. 

\begin{table}[h]
    \centering
    \resizebox{\textwidth}{!}{%
    \footnotesize
    \begin{tabular}{lllllll}
                & \textsc{Two Moons}      & \textsc{Two Circles}      & \textsc{S Curve 2d}      &\textsc{S Curve 3d}     & \textsc{Pinwheels}     & \textsc{Swissroll} \\ \toprule
    MLP (It\^o) & 38.122 (0.517)          & 33.738 (0.150)            & \textbf{54.083} (0.600)   & \textit{72.045} (0.688) & 72.333 (1.018)           &   69.345 (0.717)                 \\ 
    IM       & \textbf{37.356} (0.323) & \textbf{33.160} (0.371)   & 54.098 (0.645)            & \textbf{72.013} (0.645) &\textit{71.318} (0.985)  &  \textbf{69.000} 0.550         \\
    EM       & \textit{37.793} (0.307) & \textit{33.319 (0.264)}   & \textit{54.089} (0.607)   & 72.636 (0.600)          & \textbf{70.692 (3.596)}   & \textit{69.230} 0.546               \\ \bottomrule
    \end{tabular}
    }
    \vspace{5pt}\caption{Density estimation through linear Fokker-Planck training: ELBO between true samples and generated samples. \textbf{Bolded} values and \textit{italic} values are best and second best correspondingly.\label{tab:linearfkGen}}

\end{table}

\section{Intuition behind the proposed architectures}\label{sec:im_em_intuition}
Here we describe a few details regarding the architectures and how they can be interpreted.
The main idea is each representation has its own set of implicit biases that impart different properties on the learned drift. 

\subsection{Intuition behind the IM architecture}
\label{sec:im_intuition}
The IM architecture has two key components: the mean field layer, which acts as a weight sharing mechanism; and, the change of measure, which reweighs the shared weight.
The main idea is that by using the weight sharing component, learning the correct drift becomes ``easier'', leading to better estimates of the drift.
We represent the shared weight as a matrix $W_0$ which is of dimensions $\mathbb{R}^{K \times d}$ where $K$ is the width of the mean-field layer. 
We represent the interaction function $\varphi$ as a neural network $\varphi_\theta$.
To get an understanding how how this looks like in the stationary case (that is, $p_t = p_\star$ for all $t$), take the following illustration:
\begin{align*}
    X_t^{(i)} = \begin{bmatrix}
    x_1 \\
    x_2 \\
    \vdots \\
    x_d
    \end{bmatrix} \in \mathbb{R}^d, \quad W_0 = \begin{bmatrix}
        w_{1,1} & w_{1,2} & \cdots & w_{1,d} \\
        \vdots & \ddots & \ddots & \vdots \\
        w_{K,1} & w_{K,d} & \cdots & w_{K,d}
    \end{bmatrix} \\
    \mathrm{MF}_{(K)}(\varphi_\theta(X_t^{(i)})) = \frac1K \sum_{i=1}^K \varphi\left(\begin{bmatrix}
    x_1 \\
    x_2 \\
    \vdots \\
    x_d
    \end{bmatrix} , \begin{bmatrix}
    w_{i,1} \\
    w_{i,2} \\
    \vdots \\
    w_{i,d}
    \end{bmatrix}\right)
\end{align*}
Note that as $K\to\infty$ then this expectation can correspond to the true expectation. 
We then represent the change of measure by another neural network, $\lambda_\theta(\cdot)$.
$\lambda_\theta(\cdot)$ takes as input the weights $W_0$ and time $t$ to output the re-weighted measure.
\begin{align*}
    \mathrm{MF}_{(K)}(\varphi(X_t^{(i)}) = \frac1K \sum_{i=1}^K \varphi_\theta \left(\begin{bmatrix}
    x_1 \\
    x_2 \\
    \vdots \\
    x_d
    \end{bmatrix} , \lambda_\theta \left (\begin{bmatrix}
    w_{i,1} \\
    w_{i,2} \\
    \vdots \\
    w_{i,d}
    \end{bmatrix}, t \right ) \begin{bmatrix}
    w_{i,1} \\
    w_{i,2} \\
    \vdots \\
    w_{i,d}
    \end{bmatrix}\right)
\end{align*}

To provide more intuition on how the mean-field layer is implemented in PyTorch, it requires the \verb|repeat_interleave| operation to implement the weight sharing.
The following code snippet illustrates the module in more detail:
\begin{minted}{python}
class MeanFieldModule(nn.Module):
    def __init__(self, d):
        self.varphi = MLP(d * 2, WIDTH, DEPTH, d)  # initialize the interaction kernel network
        self.w0     = torch.randn(WIDTH_W, d)      # initialize the shared weight
        
    def forward(self, x):
        x_prime   = x.repeat_interleave(WIDTH_W, dim=0) # mean-field weight sharing
        in_varphi = torch.cat((x_prime, self.w0),dim=-1)
        B_x = self.varphi(in_varphi)                    # change of measure in self.varphi
\end{minted}

Note that the change of measure is implicit in the parameterization of \verb|varphi|.
We can adjust this as necessary, for example using a separate network that applies the change of measure to the \verb|w0| parameter. 
Having introduced this parameterization, we note that the interesting behavior comes from the implicit regularization.  

\subsection{Intuition behind the ML architecture}
The main idea behind the ML architecture is that if we explicitly represent the marginal distributions according to some density estimator, we can leverage the regularity of the density over different time steps to improve the estimation. 
We know that the flow should satisfy a nonlinear Fokker-Planck equation, so by constraining the class of densities such that they satisfy this equation, we can possibly do better than just using the empirical measure.
Taking an expectation with respect to the empirical measure --- particularly in cases where it may be corrupted by noise --- can lead to an incorrect drift.
This is particularly important when one is dealing with scientific data which may be noisy. 

Instead, we parameterize a time dependent density and ensure there is an agreement between the estimated drift and the estimated density.

\end{document}